\documentclass[11pt]{article}

\usepackage[utf8]{inputenc}
\usepackage[T1]{fontenc}
\usepackage{geometry}
\usepackage{amsmath, amssymb, amsthm}
\usepackage{xcolor}
\usepackage{graphicx}
\usepackage{booktabs}
\usepackage{microtype}
\usepackage[colorlinks=true, allcolors=blue]{hyperref}
\usepackage[numbers,sort&compress]{natbib}
\DeclareMathOperator*{\E}{\mathbb{E}}

\newcommand{\N}{\mathcal{N}}

\newcommand{\data}{\mathcal{X}}
\newcommand{\vu}{\mathbf{u}}
\newcommand{\vv}{\mathbf{v}}
\newcommand{\vx}{\mathbf{x}}

\newcommand{\ve}{\boldsymbol{\epsilon}}

\newcommand{\Emarg}{E_{\text{marg}}}        

\newtheorem{lemma}{Lemma}

\newtheorem{proposition}{Proposition}
\theoremstyle{definition}

\title{The Geometry of Noise: \\Why Diffusion Models Don't Need Noise Conditioning}
\author{Mojtaba Sahraee-Ardakan,  Mauricio Delbracio,  Peyman Milanfar \\[1.5ex]
Google} 
\date{}

\begin{document}

\maketitle

\begin{abstract}

Autonomous (noise-agnostic) generative models, such as Equilibrium Matching and blind diffusion, challenge the standard paradigm by learning a single, time-invariant vector field that operates without explicit noise-level conditioning. While recent work suggests that high-dimensional concentration allows these models to implicitly estimate noise levels from corrupted observations, a fundamental paradox remains: what is the underlying landscape being optimized when the noise level is treated as a random variable, and how can a bounded, noise-agnostic network remain stable near the data manifold where gradients typically diverge?

We resolve this paradox by formalizing Marginal Energy, $E_{\text{marg}}(\mathbf{u}) = -\log p(\mathbf{u})$, where $p(\mathbf{u}) = \int p(\mathbf{u}|t)p(t)dt$ is the marginal density of the noisy data integrated over a prior distribution of unknown noise levels. We prove that generation using autonomous models is not merely blind denoising, but a specific form of Riemannian gradient flow on this Marginal Energy. Through a novel relative energy decomposition, we demonstrate that while the raw Marginal Energy landscape possesses a $1/t^p$ singularity normal to the data manifold, the learned time-invariant field implicitly incorporates a local conformal metric that perfectly counteracts the geometric singularity, converting an infinitely deep potential well into a stable attractor.

We also establish the structural stability conditions for sampling with autonomous models. We identify a ``Jensen Gap'' in noise-prediction parameterizations that acts as a high-gain amplifier for estimation errors, explaining the catastrophic failure observed in deterministic blind models. Conversely, we prove that velocity-based parameterizations are inherently stable because they satisfy a bounded-gain condition that absorbs posterior uncertainty into a smooth geometric drift.

\end{abstract}


\section{Introduction}

Generative modeling has seen immense progress over the past decade, tracing its roots to the foundational non-equilibrium thermodynamics approach introduced by~\citep{sohl2015deep}. This paradigm was popularized and scaled by Denoising Diffusion Probabilistic Models (DDPM)~\cite{ho2020denoising} and subsequent architectural refinements \cite{dhariwal2021diffusion}, Score-based models~\cite{song2019generative,song2020improved,vincent2011connection}, and subsequently unified under the continuous-time mathematical framework of score-based stochastic differential equations (SDEs)\cite{song2020score}. For a broader conceptual overview of how these perspectives on diffusion have evolved, see \cite{dieleman2023perspectives,gao2025diffusion}. More recently, the field has evolved beyond pure diffusion to embrace velocity-based transport formulations \cite{liu2022flow,delbracio2023inversion,albergo2025stochastic}, notably Flow Matching~\cite{lipman2023flow}, as well as stationary targets like Equilibrium Matching (EqM)\cite{wang2025equilibrium}.

A defining characteristic of these standard generative models is their reliance on explicit time-conditioning. They typically learn a conditional score function or velocity field, such as $\epsilon_{\theta}(\vu,t)$, defining a dynamic field that changes with time, i.e. a field over $\mathbb{R}^{D}\times[0,1]$. In these frameworks, the network relies on the time variable $t$ to dictate the current scale of corruption and orient the trajectory.

In contrast, recent work has explored autonomous approaches, such as EqM~\cite{wang2025equilibrium} or noise-blind diffusion~\cite{song2020improved,sun2025isnoise,kadkhodaie2026blinddenoisingdiffusionmodels}, which learn a single noise-agnostic vector field $f_{\theta}(\vu)$ independent of $t$. We denote the optimal vector field learned by such a model as $f^{*}(\vu)$. It is important to distinguish the model's architecture from the resulting inference dynamics: while the sampling process may still utilize a time-dependent schedule to scale the trajectory, the underlying neural network is strictly time-invariant. This autonomous approach presents a fundamental puzzle: the ``correct'' gradient to follow from a point $\vu$ should depend heavily on its noise level. How can a single, static vector field effectively guide a sample from pure noise (high $t$) and also guide a sample from light noise (low $t$), all while ensuring its stationary points accurately reflect the clean data $\data$?

In this paper, we resolve this paradox. We show that such an autonomous model is not merely acting as a ``blind'' denoiser, but is implicitly learning a hybrid field that is fundamentally tied to a single, non-parametric marginal energy landscape ($\Emarg$). This energy is defined as the negative log-likelihood of the marginal data distribution $p(\vu) = \int p(\vu | t) p(t) dt$:
\begin{equation}
    \Emarg(\vu) = -\log \left(\int p(\vu | t) p(t) dt\right).
\end{equation}

We build our argument as follows:
\begin{enumerate}
    \item \textbf{The Energy Paradox:} We define the marginal energy and derive its explicit gradient. We rigorously show that this gradient diverges near the data manifold, creating an infinitely deep potential well that forbids stable gradient descent.
    \item \textbf{Energy-Aligned Decomposition:} We analyze the learned autonomous vector fields, proving they decompose into exactly three geometric components: a natural gradient, a transport correction (covariance) term, and a linear drift.
    \item \textbf{Riemannian Gradient Flow:} We resolve the singularity paradox by showing that noise-agnostic models implicitly implement a Riemannian gradient flow. The learned vector field incorporates a local conformal metric (the effective gain) that perfectly preconditions and counteracts the geometric singularity of the raw energy landscape.
    \item \textbf{Stability of Sampling with Autonomous Models:} We establish the mathematical conditions necessary for sampling stability. We prove that velocity-based parameterizations (e.g., Flow Matching, EqM) succeed because they absorb posterior uncertainty into a stable drift, whereas standard noise-prediction parameterizations (e.g., DDPM/DDIM) act as high-gain amplifiers for estimation errors, leading to structural instability.
    
\end{enumerate}

\section{Related Work}
\label{sec:related_work}

Our work unifies and grounds three recent lines of inquiry in generative modeling: noise unconditional generation, equilibrium dynamics, and energy-based training.

\paragraph{Noise-Blind Denoising.}
The prevailing paradigm in score-based modeling relies on conditioning the network on the noise level $t$ \citep{song2019generative}. However, \citet{sun2025isnoise} recently challenged this, demonstrating that ``blind'' models can achieve high-fidelity generation without $t$. This connects to earlier findings in image restoration by \citet{gnanasambandam2020onesize}, who showed that a single ``one-size-fits-all'' denoiser could approximate an ensemble of noise-specific estimators. Our work provides the rigorous theoretical justification for these observations, identifying $\Emarg$ as the implicit objective and connecting the gradients of this energy to the autonomous fields via concentration of the noise level given the noisy signal.

In concurrent work, \citet{kadkhodaie2026blinddenoisingdiffusionmodels} provide a highly rigorous statistical analysis of blind denoising diffusion models (BDDMs) for data with low intrinsic dimensionality. They analytically prove that under the assumption that the intrinsic dimension is much smaller than the ambient dimension ($k \ll d$), BDDMs can accurately estimate the true noise level from a single observation and implicitly track a valid noise schedule, providing robust finite-time sampling guarantees. While their work offers an exhaustive and rigorous statistical treatment of this low-dimensional data regime, our work situates this concentration of measure as a specific asymptotic case (Regime I) within a broader geometric framework. Specifically, our focus lies in connecting noise-blind generation to the gradient of the marginal energy landscape, revealing the process as a Riemannian gradient flow. Furthermore, whereas their analysis focuses primarily on specific Langevin-type SDE discretizations, our framework is generalized across arbitrary affine diffusion processes and learning targets. Finally, while they demonstrate that blind denoisers can outperform non-blind counterparts by avoiding schedule mismatch errors, our framework allows us to prove why models trained to predict noise (e.g., DDPM/DDIM) are structurally unstable for autonomous generation due to gradient singularities, demonstrating why velocity- or signal-based targets are strictly necessary.

\paragraph{Energy Landscapes \& Singularities.}
While the framework of energy-based learning is well-established \cite{lecun2006tutorial,du2019implicit}, explicitly learning energy functions is known to be unstable \citep{du2021improved}. Recent approaches like ``Dual Score Matching'' \citep{guth2025learning} attempt to stabilize this by learning a joint energy via both space and time scores. Our work analyzes the \textit{marginal} energy, aligning with \citet{scarvelis2025closed}, who proved that the exact closed-form score of a finite dataset degenerates into a nearest-neighbor lookup. While they address this by smoothing the score kernel, we show that autonomous flow models resolve it implicitly via a Riemannian preconditioner.

\paragraph{Equilibrium Dynamics \& Flow.}
\citet{wang2025equilibrium} introduced Equilibrium Matching (EqM) to replace time-dependent fields with a single time-invariant gradient. This parallels Action Matching \citep{neklyudov2023action}. Our analysis reveals that EqM is unique: it implements a natural gradient descent on the marginal energy. This connects EqM to the fundamental JKO scheme \citep{jordan1998variational}, unifying ``transport'' and ``restoration'' under a single autonomous field.

\section{Preliminaries: A Unified Schedule Formulation}
\label{sec:preliminaries}

To provide a general theory covering Diffusion Models (DDPM, EDM) and Flow Matching (EqM), we adopt the unified affine formulation proposed by \citet{sun2025isnoise}.

Let $t \in [0, 1]$ index the noise level. The noisy observation $\vu_t$ is constructed from clean data $\vx$ and noise $\ve \sim \N(\mathbf{0}, \mathbf{I})$ via method-specific schedule functions $a(t)$ and $b(t)$:
\begin{equation}
    \vu_t = a(t)\vx + b(t)\ve.
    \label{eq:unified_forward}
\end{equation}
Assuming that the data is normalized, the signal-to-noise ratio (SNR) at time $t$ is
\begin{equation}
    \text{SNR} = \frac{a^2(t)}{b^2(t)}.\label{eq:SNR}
\end{equation}
Generative models are typically trained to predict a linear target $r(\vx, \ve, t) = c(t)\vx + d(t)\ve$ by minimizing the Mean Squared Error (MSE):
\begin{equation}
    \mathcal{L}(f) = \E_{\vx, \ve, t} \left[ \| f(\vu_t) - (c(t)\vx + d(t)\ve) \|^2 \right].
    \label{eq:mse_loss}
\end{equation}

In the standard diffusion paradigm, $f$ is explicitly conditioned on the noise level $t$, i.e. it has the form\footnote{While a more accurate notation would have been $f(\vu, t)$, with some abuse of notation we use $f_t(\vu)$ to denote a function of $\vu$ and $t$ to keep the notation consistent with the autonomous models.} $f_t(\vu)$. The minimizer of the MSE loss is then the conditional expectation of the target:
\begin{equation}
    f^*_t(\vu) = \E_{\vx, \ve | \vu, t} [c(t)\vx + d(t)\ve].
\end{equation}
This function defines a time-dependent vector field that guides the generation process.

In this work, however, we focus on autonomous models where the network $f(\vu)$ receives only the noisy observation $\vu$, with no access to $t$. The minimizer of the MSE loss for such a ``noise-agnostic'' model is the \textit{posterior expectation} of the target (Lemma~\ref{lemma:optimal_target} in Appendix~\ref{sec:optimal_autonomous_target}):
\begin{equation}
    f^*(\vu) = \E_{t|\vu} \left[ \E_{\vx, \ve | \vu, t} [c(t)\vx + d(t)\ve] \right] = \E_{t|\vu} [f^*_t(\vu)].
\end{equation}
This has a very intuitive interpretation: \textit{the optimal autonomous model is a time-average of the optimal conditional model with respect to the posterior $p(t|\vu).$}

By defining the optimal conditional denoiser as $D_t^*(\vu) = \E[\vx | \vu, t]$, we can expand this target as (Lemma~\ref{lemma:denoiser_formulation} in Appendix \ref{sec:optimal_autonomous_target}):
\begin{equation}
    f^{*}(\vu) = \mathbb{E}_{t|\vu} \left[ \frac{d(t)}{b(t)}\vu + \left( c(t) - \frac{d(t)a(t)}{b(t)} \right) D_{t}^{*}(\vu) \right].
    \label{eq:autonomous_target_expansion}
\end{equation}
Specific choices for the coefficients $a(t), b(t), c(t), d(t)$ yield standard architectures, as summarized in Table~\ref{tab:unified_schedules}. In the table, DDPM formulation uses the so called variance preserving diffusion processes where coefficients $a(t)$ and $b(t)$ satisfy $a^2(t) + b^2(t) = 1$ and $a(t) =\sqrt{\bar{\alpha}_t}$ is defined via a discrete diffusion process that determine the form of $\bar{\alpha}_t$ \cite{ho2020denoising,nichol2021improved}.

\emph{The central question of this paper is to understand the geometric and dynamical consequences of replacing the precise conditional field $f^*_t(\vu)$ with this autonomous posterior average $f^*(\vu)$. Does this time-invariant field still define a valid generative trajectory?}

In the following sections, we analyze its alignment with an energy landscape (Section~\ref{sec:marginal_energy}) and deep dive into its properties when used as a generative model (Sections~\ref{sec:riemannian_flow} and \ref{sec:stability}).

\begin{table*}[t]
    \centering
    \caption{Unified coefficients and simplification of the general autonomous field for common generative models: DDPM~\cite{ho2020denoising}, EDM~\cite{karras2022elucidating}, Flow Matching (FM)~\cite{lipman2023flow}, and Equilibrium Matching (EqM) ~\cite{wang2025equilibrium}.}
    \label{tab:unified_schedules}
    \renewcommand{\arraystretch}{2.2}
    \begin{tabular}{lccccc}
        \toprule
        \textbf{Model} & $a(t)$ & $b(t)$ & $c(t)$ & $d(t)$ & \textbf{Autonomous field} $f^{*}(\vu)=\mathbb{E}_{t|\vu} \left[f^*_t(\vu)\right]$ \\
        \midrule
        \textbf{DDPM} & $\sqrt{\bar{\alpha}_t}$ & $\sqrt{1-\bar{\alpha}_t}$ & 0 & 1 & $\mathbb{E}_{t|\vu} \left[ \frac{\vu - \sqrt{\overline{\alpha}_{t}}D_{t}^{*}(\vu)}{\sqrt{1-\overline{\alpha}_{t}}} \right] = \mathbb{E}_{t|\vu}[\epsilon_{t}^{*}(\vu)]$ \\
        \textbf{EDM} & 1 & $\sigma_t$ & 1 & 0 & $\mathbb{E}_{t|u} \left[ D_{t}^{*}(\vu) \right]$ \\
        \textbf{FM} & $1-t$ & $t$ & $-1$ & $1$ & $\mathbb{E}_{t|\vu} \left[ \frac{\vu - D_{t}^{*}(\vu)}{t} \right]$ \\
        \textbf{EqM} & $1-t$ & $t$ & $-t$ & $t$ & $\mathbb{E}_{t|u} \left[ \vu - D_{t}^{*}(\vu) \right]$ \\
        \bottomrule
    \end{tabular}
\end{table*}

\section{The Geometry of the Marginal Energy}
\label{sec:marginal_energy}

Standard diffusion models rely on a time-dependent score function, $\nabla_\vu \log p(\vu|t)$, which explicitly guides the trajectory at every noise level. In contrast, \textbf{autonomous models} (such as Equilibrium Matching or ``blind'' diffusion) must compress these dynamics into a single, static vector field $f^*(\vu)$ that is independent of time.

This fundamental difference raises a critical geometric question: \textit{Does this static field align with the gradient of a global potential energy?} If such a potential exists, \textbf{autonomous generation} could be theoretically grounded as a form of energy minimization \cite{lecun2006tutorial}.

The most natural candidate for this potential is the \textit{marginal energy} $\Emarg(\vu)$, defined as the negative log-likelihood of the marginal data density $p(\vu) = \int p(\vu|t)p(t)dt$:
\begin{equation}
    \Emarg(\vu) = - \log p(\vu).
\end{equation}
To determine if the learned field $f^*(\vu)$ aligns with this energy, we must first derive its gradient. By differentiating the marginal likelihood mixture, we find that the gradient of the marginal energy is the posterior expectation of the conditional scores:
\begin{equation}
    \nabla_\vu \Emarg(\vu) =\E_{t|\vu}[-\nabla_\vu \log p(\vu|t)].
\end{equation}
We refer the reader to Lemma~\ref{lemma:marginal_grad} for the proof. To evaluate this expectation, we use Tweedie's formula \cite{robbins1992empirical,efron2011tweedie} to express the conditional score in terms of the optimal denoiser:
\begin{equation}
    \nabla_\vu \log p(\vu|t) = \frac{a(t)D_t^*(\vu) - \vu}{b(t)^2}.
\end{equation}
Substituting this directly into the posterior expectation yields the explicit form of the marginal energy gradient:
\begin{equation}
    \nabla_\vu \Emarg(\vu) = \E_{t|\vu} \left[ \frac{\vu - a(t) D_t^*(\vu)}{b(t)^2} \right].
    \label{eq:marginal_grad}
\end{equation}
This result establishes the link between the static geometry and the dynamic denoising process. However, it also exposes a critical flaw in the landscape itself.

\subsection{The Energy Paradox}
A key requirement for generative modeling is that the learned vector field $f^*(\vu)$ must be consistent with the clean data support. Depending on the specific formulation, the field at the boundary ($t \to 0$) generally behaves in one of two ways:

\paragraph{Case 1: Attractors (EqM, EDM).} For equilibrium-based models, the target is the data itself or a restoration term. Here, the ideal field must vanish at the clean data ($f^*(\vx) = \mathbf{0}$) to create a stable fixed point.

\paragraph{Case 2: Transversal Flows (Flow Matching).} For transport-based models, the target is a velocity vector (e.g., $\vx_1 - \vx_0$). At the data, the field does not vanish but converges to a finite, non-zero velocity vector ($f^*(\vx) \approx \mathbf{0} - \vx$) that ensures the trajectory intersects the data manifold at the correct time.

\paragraph{The Singularity.}
Regardless of whether the field acts as an attractor or a transversal flow, the model faces a \textit{geometric singularity}. As established in Appendix~\ref{sec:proof_concentration}, the posterior $p(t|\vu)$ collapses near the data manifold. The term inside the expectation in Equation \eqref{eq:marginal_grad} has a singularity as $t\rightarrow 0$. Consequently, the marginal energy forms an infinitely deep potential well ($\Emarg \to -\infty$, see Figure~\ref{fig:marginal_energy}), causing the associated gradient field to diverge:
\begin{equation}
    \lim_{\vu \to \vx_k} \| \nabla_\vu \Emarg(\vu) \| = \infty.
\end{equation}
This creates a puzzle: how can a neural network learn a bounded vector field (which must be finite at the data) that aligns with a geometry defined by such a singular potential? 

One might argue that in practice, training is stabilized by truncating the noise level at some $t_{min} > 0$ (the ``ill-conditioned regime''). However, this does not resolve the geometric paradox; it merely converts a mathematical singularity into an extremely stiff optimization landscape where Hessian eigenvalues scale as $1/t_{min}^2$. Whether strictly singular or merely ill-conditioned, the raw energy landscape essentially forbids stable gradient descent.

In the next section, we resolve this puzzle by showing that noise-blind models do not follow the raw energy gradient, but rather a \textit{Riemannian gradient flow} that perfectly preconditions this singularity.

\begin{figure}
\centering
\includegraphics[width=.85\textwidth]{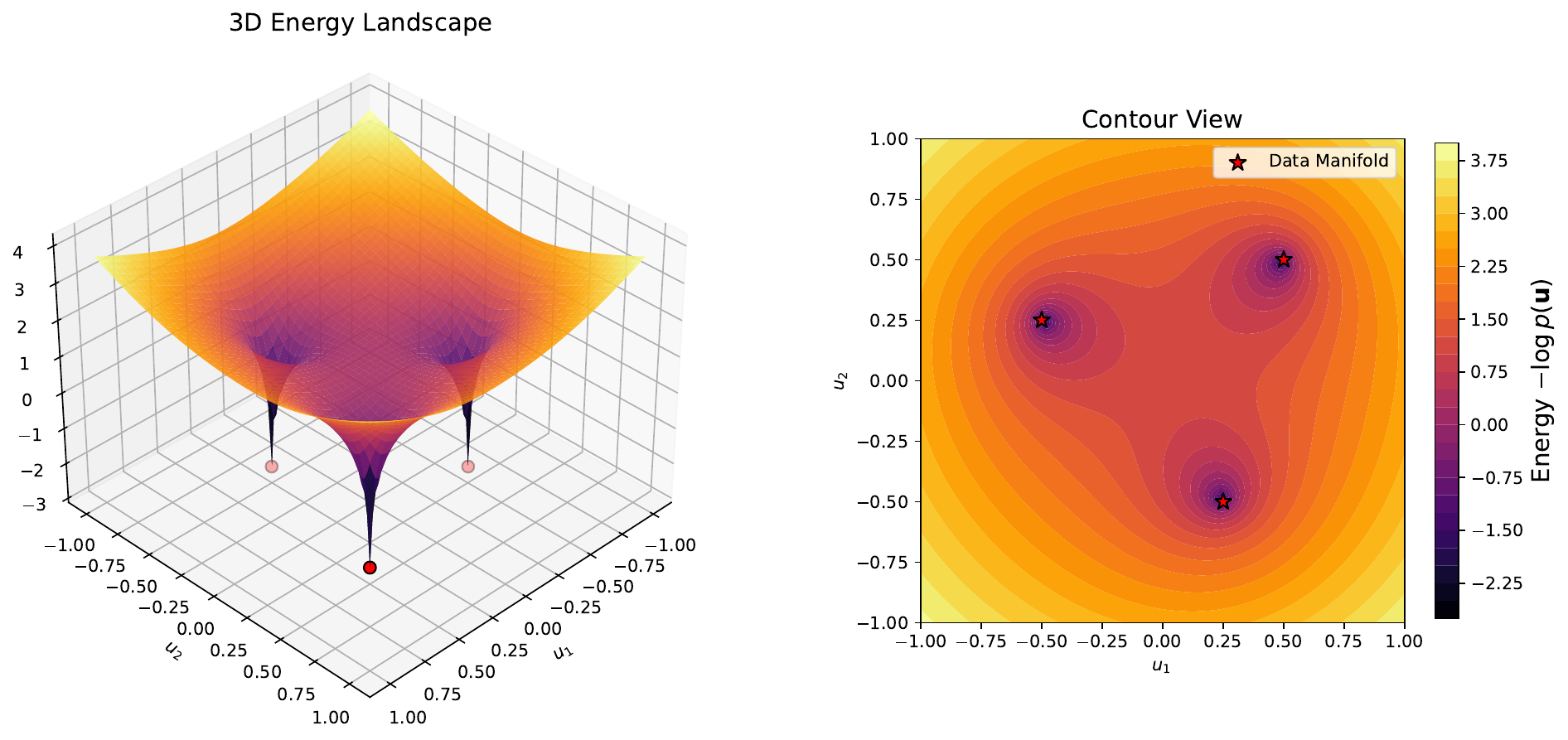}
\caption{\textbf{The Singular Geometry of the Marginal Energy Landscape.} (Left) 3D Energy Landscape: A visualization of the marginal energy $E_{marg}(u) = -\log p(u)$. The landscape reveals an infinitely deep potential well at the data manifold, where the energy diverges to $-\infty$.
(Right) Contour View: Top-down perspective showing the energy concentration around discrete data points (stars). While the raw gradient $\nabla_u E_{marg}(u)$ becomes singular as $u$ approaches the clean data, in this work we prove that autonomous models remain stable by implicitly implementing a Riemannian gradient flow. In this framework, the posterior noise variance acts as a local conformal metric that preconditions and perfectly counteracts the geometric singularity.}
\label{fig:marginal_energy}
\end{figure}

\section{Autonomous Generation as Riemannian Gradient Flow}
\label{sec:riemannian_flow}

We resolve the paradox of stable convergence despite divergent gradients by showing that autonomous models implement a Riemannian gradient flow. We show that the learned vector field $f^*(\vu)$ is structurally identical to the natural gradient of the marginal energy, but with a critical correction term that dominates only when geometric concentration fails.

\subsection{The Anatomy of the Autonomous Field}

Recall from Lemma~\ref{lemma:denoiser_formulation} that the optimal autonomous field is given by:
\begin{equation}
    f^{*}(\vu) = \mathbb{E}_{t|\vu} \left[ \frac{d(t)}{b(t)}\vu + \frac{c(t)b(t) -d(t)a(t)}{b(t)} D_{t}^{*}(\vu) \right].
    \label{eq:field_recall}
\end{equation}
This formulation reveals that for any affine schedule, the autonomous vector field is driven by two competing forces:
\begin{enumerate}
    \item \textbf{The Repulsive Expansion} ($\frac{d(t)}{b(t)}\vu$): A linear term scaling with the noise geometry. Assuming standard diffusion signs ($d, b > 0$), this pushes trajectories outward, accounting for the expanding volume of the noise distribution.
    \item \textbf{The Relative Restoration}: A denoising term pulling the sample towards high-density data regions. Its magnitude is modulated by the signal-to-noise ratio (defined in Equation \eqref{eq:SNR}) and the specific target coefficients $c(t), d(t)$.
\end{enumerate}

\paragraph{The Energy-Aligned Decomposition.}
To understand how these forces align with the geometry of the marginal energy $\Emarg$, we apply the properties of covariance to Eq. \eqref{eq:field_recall}. As derived in Appendix \ref{app:covariance_decomposition}, the vector field decomposes into exactly three geometric components:
\begin{equation}
    f^{*}(\vu) = \underbrace{\overline{\lambda}(\vu)\nabla \Emarg(\vu)}_{\text{Natural Gradient}} + \underbrace{\mathbb{E}_{t|u} \left[\left(\lambda(t) - \overline{\lambda}(\vu)\right)\left(\nabla E_{t}(\vu) -\nabla \Emarg(\vu) \right)\right]}_{\text{Transport Correction}} + \underbrace{\overline{c}_{scale}(\vu)\vu}_{\text{Linear Drift}},
    \label{eq:exact_covariance_decomposition}
\end{equation}
where $\overline{c}_{scale}(\vu) \triangleq \mathbb{E}_{t|u} [ c(t)/a(t) ]$ is the mean drift, $\lambda(t)$ is the \textit{effective gradient gain} and $\overline{\lambda}(\vu)$ is its time average defined as:
\begin{equation}
    \lambda(t) \triangleq \frac{b(t)}{a(t)} \left( d(t)a(t) - c(t)b(t) \right), \quad \overline{\lambda}(\vu) \triangleq \mathbb{E}_{t|u} [\lambda(t)].
\end{equation}
This decomposition (Eq.\ref{eq:exact_covariance_decomposition}) resolves the paradox by isolating the model's behavior into interpretable geometric terms. The field is a Riemannian flow \cite{amari1998natural,ambrosio2005gradient} (via the gain $\lambda$) modified by a Transport Correction (covariance) term. 

Importantly, this formulation exposes the mechanism of stability. While both the marginal energy gradient $\nabla \Emarg$ and the conditional energy gradient $\nabla E_t$ become singular near the manifold (diverging as $O(1/b(t))$), as shown in detail in Appendix \ref{sec:asymptotics}, the effective gain $\lambda(t)$ acts as a \textit{perfect preconditioner}. It vanishes at a rate that exactly counteracts the divergence of the gradients, ensuring the product remains bounded.

In what follows, we show that the correction term vanishes in two key asymptotic regimes: \textit{global high-dimensional concentration} (Sec. \ref{sec:concentration_main}) and \textit{local near-manifold proximity} (Sec. \ref{sec:geometric_regimes}). In both limits, the posterior $p(t|\vu)$ concentrates, simplifying the dynamics to a pure, preconditioned natural gradient flow.

\subsection{Regime I: Global Concentration in High Dimensions}
\label{sec:concentration_main}

A central mystery of autonomous models is how a single vector field can ``know'' which noise level to apply to a given input without explicit conditioning. We resolve this by observing that in high-dimensional spaces ($D \gg 1$), provided the data resides on a low-dimensional manifold ($d \ll D$), the noise level $t$ is not truly hidden; it is globally encoded in the geometry of the observation $\vu$.

In these high-dimensional settings, the mass of a Gaussian distribution concentrates in a thin spherical shell~\cite{ledoux2001concentration,vershynin2018high}. When the data is low-dimensional, the noisy observation $\vu$ can be decomposed into a component within the data subspace and an orthogonal noise component. Because the codimension is large, the magnitude of this orthogonal noise dominates the total norm. In this regime, the ``shells'' corresponding to different noise levels $b(t)$ become effectively disjoint. For completeness we provide a proof of this in Appendix~\ref{app:concentration_proof}.

As a result, the input $\vu$ becomes a deterministic proxy for the noise level $t$. This geometric structure has two profound consequences for the model's objective:
\begin{itemize}
    \item \textbf{Posterior Concentration:} The model's uncertainty about the noise level vanishes. The posterior $p(t|\vu)$ collapses to a Dirac delta centered at an implicit estimate $\hat{t}(\vu)$.
    \item \textbf{Vanishing Transport Correction:} Because there is no longer a mixture of conflicting noise levels at any given point, the complex interaction between different potential fields disappears. The transport correction term vanishes:
    \begin{equation}
        \mathbb{E}_{t|u} \left[\left(\lambda(t) - \overline{\lambda}(\vu)\right)\left(\nabla E_{t}(\vu) -\nabla \Emarg(\vu) \right)\right] \to \mathbf{0}.
    \end{equation}
\end{itemize}
Consequently, in high dimensions, the field is strictly dominated by the \textbf{Natural Gradient} flow:
\begin{equation}
    f^*(\vu) \approx \overline{\lambda}(\vu)\nabla \Emarg(\vu) + \overline{c}_{scale}(\vu)\vu.
\end{equation}
This resolves the ``blindness'' paradox globally: the model implicitly sees $t$ through the separation of noise scales. 

\subsection{Regime II: Local Stability via Proximity}
\label{sec:geometric_regimes}

While high-dimensional concentration provides a global mechanism, a second, stronger mechanism ensures stability as the trajectory approaches the data, \emph{regardless of dimension}. We analyze the decomposition in the near-manifold limit:

\paragraph{1. The Near-Manifold Regime (Concentration via Proximity)}
As the observation $\vu$ approaches the data support ($\vu \to \mathcal{X}$), the likelihood becomes dominated by the smallest noise scales. This causes the posterior $p(t|\vu)$ to concentrate sharply on $t \rightarrow 0$ simply because the observation is indistinguishable from clean data.

We rigorously prove this in Appendix \ref{sec:proof_concentration} for two cases. In the first case, we assume that the data is discrete and finite. In this case, we show that as we approach a data point, the posterior $p(t|\vu)$ converges weakly to the Dirac measure $\delta(t)$. In the second case, we assume that the data lies on a manifold of dimension $d$ in an ambient space of dimension $D$ where $D-d > 2$. Note that we do not require $D \gg d$ as in Section~\ref{sec:concentration_main}. In this case as in the case for discrete data, we can show the weak convergence of $p(t|\vu)$ as we approach the data manifold.

Notably, this \textit{local concentration} occurs even in low dimensions. Therefore, as we approach the data manifold, the transport correction term in the autonomous field becomes negligible:
\begin{equation}
    \mathbb{E}_{t|u} \left[\left(\lambda(t) - \overline{\lambda}(\vu)\right)\left(\nabla E_{t}(\vu) -\nabla \Emarg(\vu) \right)\right] \to \mathbf{0}.
\end{equation}
In this limit, the raw energy gradients ($\nabla \Emarg$ and $\nabla E_t$) diverge at an $O(1/b(t))$ rate, creating a potential geometric singularity. However, the field remains stable because the effective gain implements a Riemannian preconditioning:
\begin{itemize}
    \item \textbf{Geometric Preconditioning:} The effective gain $\overline{\lambda}(\vu)$ vanishes at a rate that exactly matches the divergence of the gradients ($\nabla \Emarg$ and $\nabla E_t$), neutralizing the infinity.
    \item \textbf{Singularity Absorption:} In transport-based models (e.g., Flow Matching), the linear drift term acts as a counter-force that effectively ``absorbs'' any residual singular component of the energy gradient, resulting in a smooth, finite velocity.
\end{itemize}

\paragraph{2. The High-Noise Regime (Transport Dominated)}
Far from the data manifold, if the dimension $D$ is not sufficiently large to enforce global concentration (Sec. \ref{sec:concentration_main}), the strict proximity cue is absent. Here, the covariance term becomes significant, ``steering'' the trajectories away from the raw energy gradient. This rotation ensures the field satisfies the global transport requirements of the noise schedule before the local geometry takes over.

For the exact low-noise asymptotic derivations and verification that architectures like Equilibrium Matching and Flow Matching satisfy these bounded field conditions, see Appendix~\ref{sec:asymptotics}.

While we showed that the target field is bounded and very close to the optimal conditional field in certain regimes, the dynamics used to generate samples are sensitive to the integrator's step size and coefficients. A bounded target divided by a vanishing noise scale creates a stiff differential equation resulting in instabilities. We discuss this in the next section.

\section{Stability Conditions for Sampling with Autonomous Models}
\label{sec:stability}

While Section~\ref{sec:riemannian_flow} established that the optimal autonomous target $f^{*}(\vu)$ is geometrically well-behaved and acts as an accurate proxy for the conditional field in regimes of concentration (high dimensions or near-manifold), this does not guarantee stable generation. The \textit{dynamics} of the sampling process can amplify small errors into divergent trajectories.

To quantify this, we analyze the sampling process as the integration of a time-dependent velocity field $\vv(\vu, t)$:
\begin{equation}
    \frac{d\vu}{dt} = \vv_{aut}(\vu, t) \triangleq \mu(t)\vu + \nu(t) f^*(\vu).
\end{equation}
Here, $\mu(t)$ is the drift coefficient of the noise schedule and $\nu(t)$ is the \textbf{effective gain} of the parameterization (derived in Appendix~\ref{sec:stability_details}). Note that even though the autonomous model $f^*(\vu)$ is time-independent, the sampler velocity $\vv_{aut}(\vu, t)$ remains a function of time because the schedule coefficients $\mu(t)$ and $\nu(t)$ vary during integration.

We compare this autonomous velocity against an ideal ``Oracle'' sampler that has access to the exact noise level $t$. The structural stability is determined by the \textit{Drift Perturbation Error} $\Delta \vv$, which measures the deviation caused by substituting the conditional target with the autonomous approximation:
\begin{align}
    \vv_{orc}(\vu, t) &= \mu(t)\vu + \nu(t) f^*_t(\vu) \\
    \vv_{aut}(\vu, t) &= \mu(t)\vu + \nu(t) f^*(\vu)
\end{align}
Subtracting the two eliminates the linear term, isolating the error introduced by the target parameterization:
\begin{equation}
    \Delta \vv(\vu, t) \triangleq \| \vv_{aut}(\vu,t) - \vv_{orc}(\vu, t) \| = \underbrace{|\nu(t)|}_{\text{Gain}} \cdot \underbrace{\| f^*(\vu) - f^*_t(\vu) \|}_{\text{Estimation Error}}.
    \label{eq:drift_error_main}
\end{equation}
This decomposition reveals that stability is a race condition as $t \to 0$: the estimation error (posterior uncertainty) naturally tends to zero, but the effective gain $\nu(t)$ may diverge. We analyze this competition for three standard parameterizations. Detailed derivations can be found in Appendix~\ref{sec:stability_details}:

\begin{itemize}
    \item \textbf{Noise Prediction (DDPM/DDIM):} The effective gain scales inversely with the noise standard deviation ($\nu(t) \propto 1/b(t)$). As $t \to 0$, this singularity amplifies the finite \textit{``Jensen Gap''} as defined in Equation \eqref{eq:amplification_factor}—the mismatch between the harmonic mean of noise levels and the true noise level—causing the error to diverge ($\lim \Delta \vv \to \infty$).
    
    \item \textbf{Signal Prediction (EDM):} The gain contains a stronger singularity ($\nu(t) \propto 1/b(t)^2$). However, the error in the signal estimator vanishes \textit{exponentially} fast near the discrete data manifold. This rapid convergence counteracts the polynomial divergence of the gain, resulting in a stable flow ($\lim \Delta \vv \to 0$).
    
    \item \textbf{Velocity Prediction (Flow Matching):} The update is identity-mapped with a bounded gain ($\nu(t) = 1$). There are no singular coefficients to amplify errors. The dynamics absorb posterior uncertainty into a bounded effective drift, making this parameterization inherently stable.
\end{itemize}

Table~\ref{tab:stability_summary} summarizes these regimes. It is important to note that this analysis identifies \textit{sufficient conditions for instability}. A divergence ($\Delta \vv \to \infty$) guarantees failure, whereas a bounded error is a necessary (but not strictly sufficient) condition for high-fidelity generation. Our results prove that velocity-based parameterizations satisfy this necessary condition, whereas noise prediction structurally fails for autonomous models.

\begin{table}[h]
    \centering
    \caption{Summary of Stability Analysis for Autonomous Models. The Drift Perturbation Error is the product of the Effective Gain $\nu(t)$ and the estimation error. Detailed derivations are provided in Appendix~\ref{sec:stability_details}.}
    \label{tab:stability_summary}
    \renewcommand{\arraystretch}{1.5}
    \begin{tabular}{lccc}
        \hline
        \textbf{Parameterization} & \textbf{Effective Gain} $\nu(t)$ & \textbf{Error Mechanism} & \textbf{Stability} \\
        \hline
        Noise ($\ve$) & $O(1/b(t))$ & Amplified Jensen Gap & \color{red}\textbf{Unstable} \\
        Signal ($\vx$) & $O(1/b(t)^2)$ & Exp. Decay vs. Poly. Div. & \color{green!60!black}\textbf{Stable} \\
        Velocity ($\vv$) & $\mathbf{1}$ (Bounded) & Bounded Drift & \color{green!60!black}\textbf{Inherently Stable} \\
        \hline
    \end{tabular}
\end{table}

\section{Empirical Verification}
\label{sec:empirical_verification}

To validate the theoretical stability conditions derived in Section~\ref{sec:stability}, we conducted experiments on CIFAR-10, SVHN and Fashion MNIST datasets. The primary objective was to determine if the predicted structural instability of autonomous noise-prediction models (DDPM Blind) manifests in standard image benchmarks, and whether velocity-based parameterizations (Flow Matching) can resolve this paradox without explicit noise conditioning.

\subsection{Experimental Setup}
We trained four model configurations using a ResNet-based U-Net architecture. All models were trained for 10,000 steps using EMA=0.999 and batch size=128.
\begin{itemize}
\item \textbf{DDPM Blind (Autonomous):} A noise-prediction model where time-level conditioning is removed.
   \item \textbf{DDPM Conditional:} The standard baseline utilizing explicit time embeddings.
   \item \textbf{Flow Matching Blind (Autonomous):} A velocity-parameterized model ($v = \dot{u}$) without noise-level conditioning.
    \item \textbf{Flow Matching Conditional:} A velocity-based model with explicit $t$ conditioning.
\end{itemize}

\paragraph{Findings: Parameterization and Stability.} The generative results align with our theoretical stability analysis. 

\begin{itemize}
    \item \textbf{Unstable Noise Prediction:} As predicted in Section \ref{sec:stability}, the \textit{DDPM Blind} model fails to generate coherent samples. The resulting images are dominated by high-frequency artifacts and residual noise, confirming that the $O(1/b(t))$ gain singularity in noise-prediction acts as an amplifier for estimation errors.
    \item \textbf{Stable Velocity Flows:} In contrast, the \textit{Flow Matching Blind} model produces sharp samples qualitatively similar to its conditional counterpart. Because the gain remains bounded ($\nu(t)=1$), the dynamics absorb posterior uncertainty into a stable effective drift.
\end{itemize}

These findings demonstrate that while the marginal energy landscape contains a fundamental singularity, velocity-based architectures remain stable by implicitly implementing a Riemannian gradient flow that preconditions the landscape.

\begin{figure}[h]
    \centering
    \includegraphics[width=.8\textwidth]{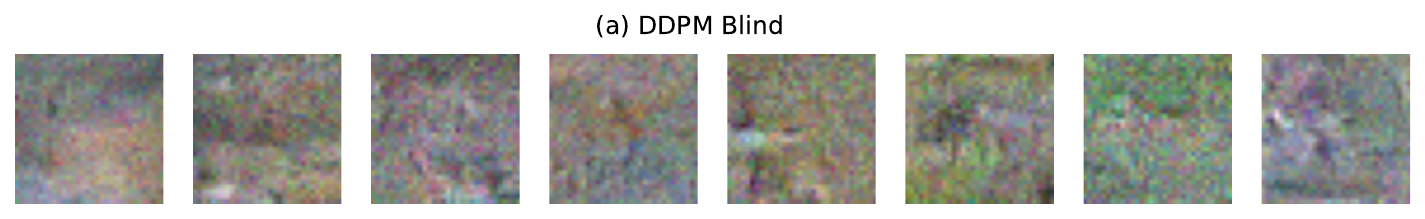}
    \includegraphics[width=.8\textwidth]{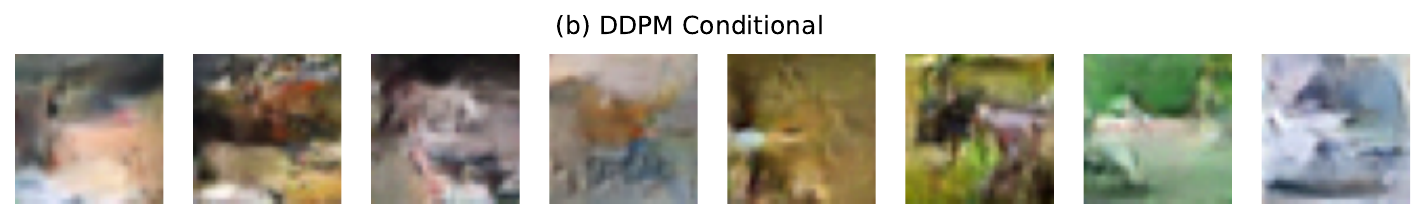}
    \includegraphics[width=.8\textwidth]{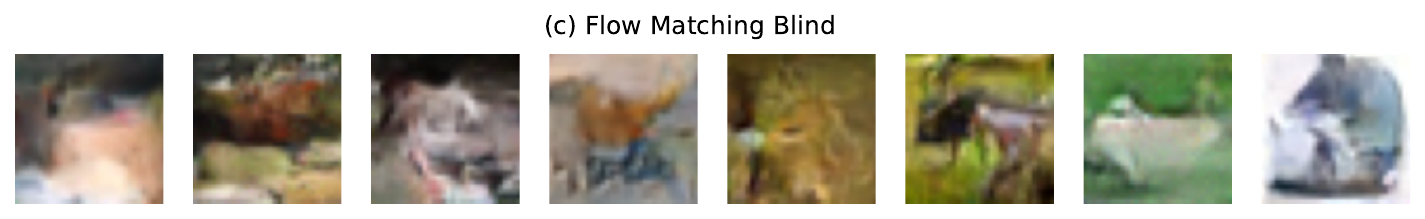}
    \includegraphics[width=.8\textwidth]{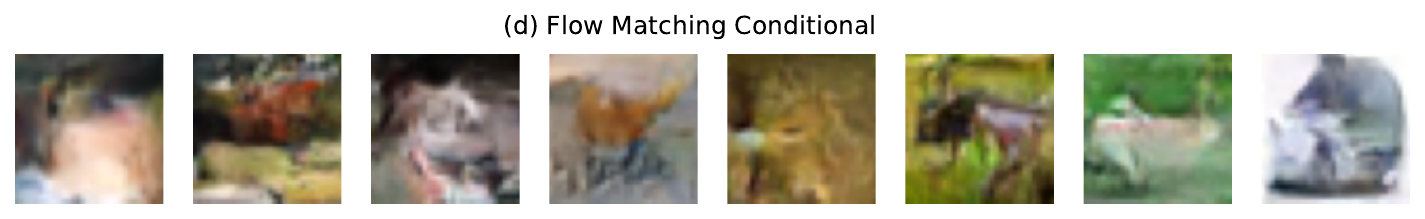}
    \caption{\textbf{Generative performance on CIFAR-10.} Top: DDPM Blind exhibits structural instability and noise. Bottom: Flow Matching Blind achieves stable generation, matching the performance of conditioned models.}
    \label{fig:cifar_results}
\end{figure}

\begin{figure}[h]
    \centering
    \includegraphics[width=.8\textwidth]{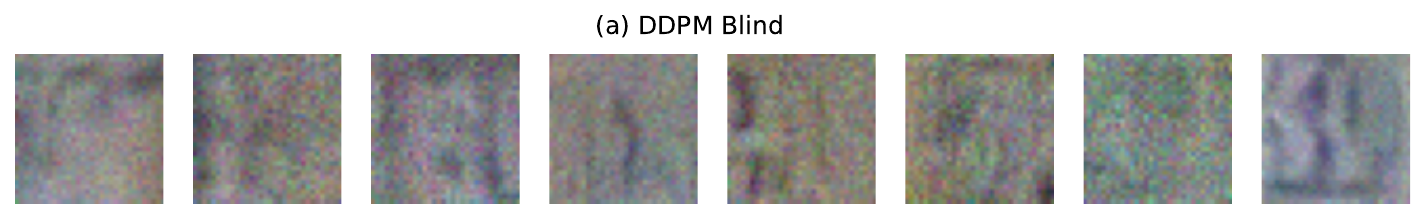}
    \includegraphics[width=.8\textwidth]{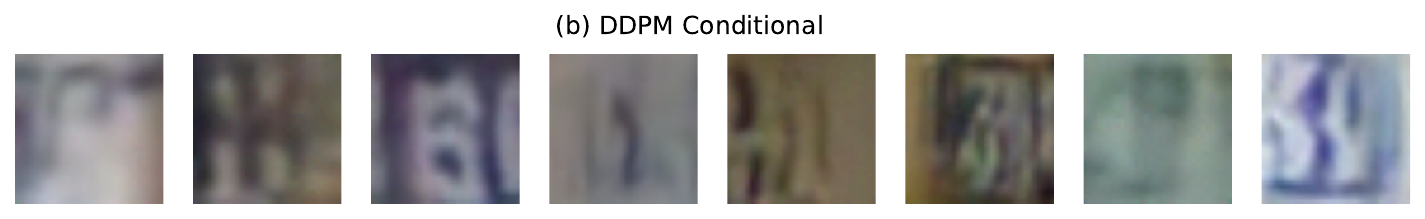}
    \includegraphics[width=.8\textwidth]{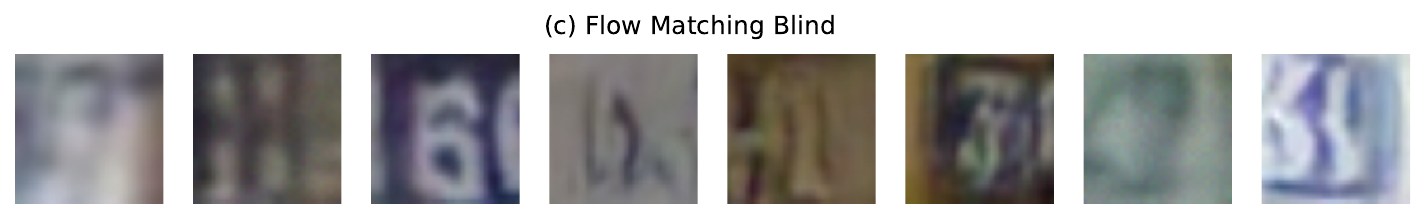}
    \includegraphics[width=.8\textwidth]{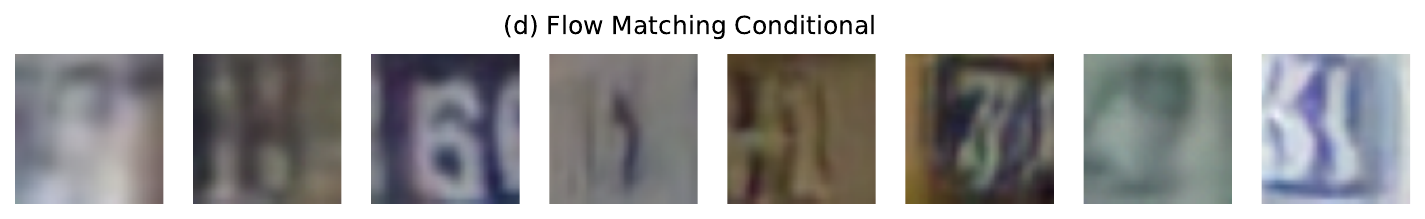}
    \caption{\textbf{Generative performance on SVHN (Street View House Numbers)}. Top: DDPM Blind exhibits structural instability and noise. Bottom: Flow Matching Blind achieves stable generation, matching the performance of conditioned models.}
    \label{fig:svhn_results}
\end{figure}

\begin{figure}[h]
    \centering
    \includegraphics[width=.8\textwidth]{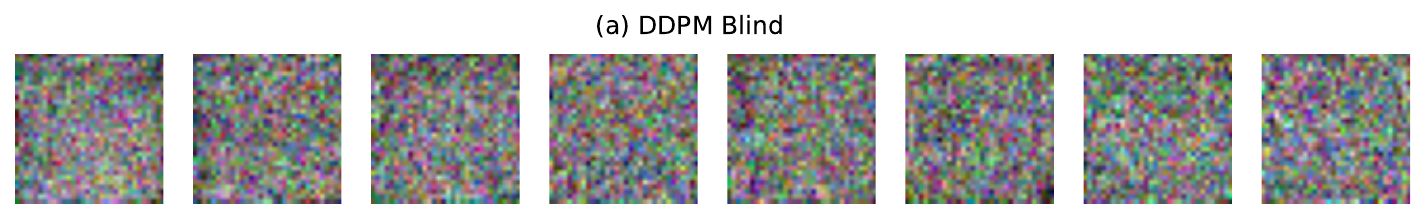}
    \includegraphics[width=.8\textwidth]{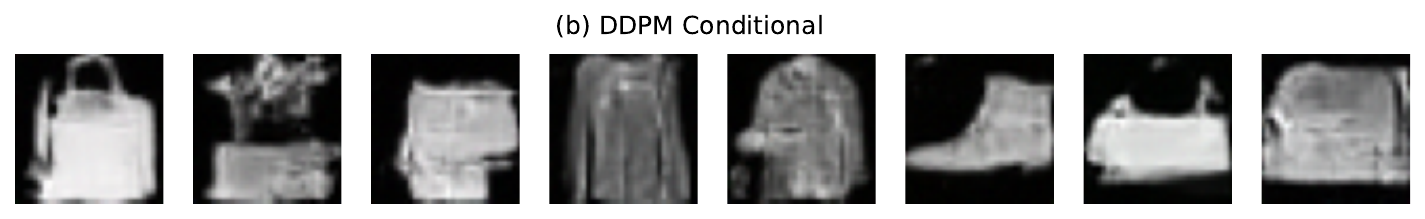}
    \includegraphics[width=.8\textwidth]{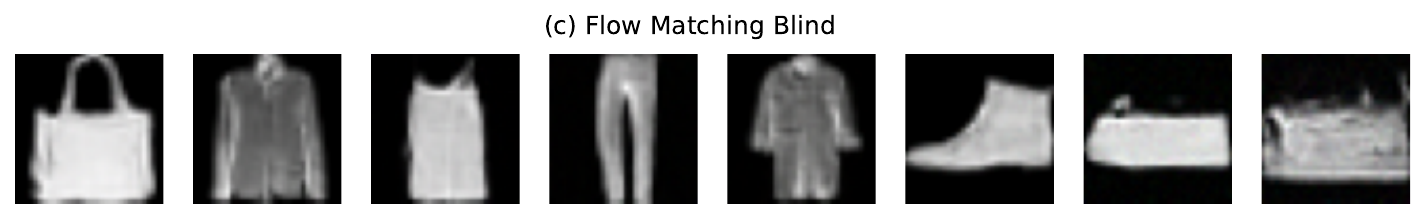}
    \includegraphics[width=.8\textwidth]{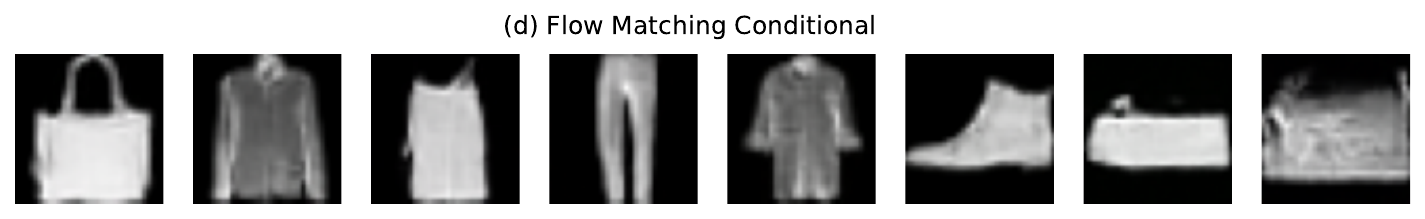}
    \caption{\textbf{Generative performance on Fashion MNIST.}. Top: DDPM Blind exhibits structural instability and noise. Bottom: Flow Matching Blind achieves stable generation, matching the performance of conditioned models.}
    \label{fig:fashion_mnist_results}
\end{figure}

\subsection{The Impact of Dimensionality on Autonomous Generation}
To empirically illustrate how high-dimensional geometry resolves the ambiguity of autonomous generation, we designed a controlled toy experiment motivated by the setup in~\cite{li2025back}. We constructed a 2D concentric circles dataset and embedded it into a high-dimensional ambient space $\mathbb{R}^D$ using a random orthogonal projection matrix $P \in \mathbb{R}^{D \times 2}$, where $P^T P = I$. We trained a standard residual network for both Flow Matching and DDPM under conditional and autonomous (``blind'') settings. For the conditional variants, the network received the true time embedding $t$, whereas for the autonomous variants, $t$ was strictly zeroed out, forcing the network to implicitly infer the noise scale from the spatial coordinates alone.

Figure~\ref{fig:toy-circles} visualizes the generated samples projected back down to the 2D subspace across exponentially increasing ambient dimensions ($D \in \{2, 8, 32, 128\}$). The results highlight three distinct geometric regimes that perfectly mirror our theoretical stability analysis:

\begin{itemize}
\item The low-dimensional ambiguity regime ($D = 2$). In low dimensions, both \textbf{autonomous models} struggle to capture the true distribution. Because the noise shells heavily overlap, the posterior noise distribution $p(t|u)$ is highly ambiguous, resulting in diffuse, noisy sampling. The network lacks the geometric cues necessary to separate noise scales.

\item The parameterization stability regime ($D \in \{8, 32\}$). As the ambient dimension increases, probability mass begins to concentrate into disjoint shells, giving the network implicit cues about the noise scale. In these moderate dimensions, both models successfully begin to resolve the global ring structure. However, the structural stability of the underlying parameterization dictates the precision of the generated samples. Autonomous Flow Matching (FM Blind) leverages its bounded velocity target to smoothly absorb residual posterior uncertainty, resulting in tight, clean concentric circles as early as $D=8$. In contrast, DDPM Blind exhibits noticeably higher variance and background scatter. This empirically demonstrates that the $O(1/b(t))$ gain in noise-prediction architectures acts as an amplifier for residual estimation errors, leading to noisier sampling trajectories before absolute concentration is reached.

\item The absolute concentration regime ($D = 128$). In extreme high dimensions, the geometric concentration becomes so sharp that the posterior $p(t|u)$ effectively collapses to a Dirac delta. Consequently, the network's estimation error of the noise scale vanishes. Because the estimation error drops to zero faster than the DDPM gain diverges, even the structurally unstable DDPM Blind model eventually produces clean, coherent samples.
\end{itemize}

\begin{figure}
\centering
\includegraphics[width=.95\textwidth]{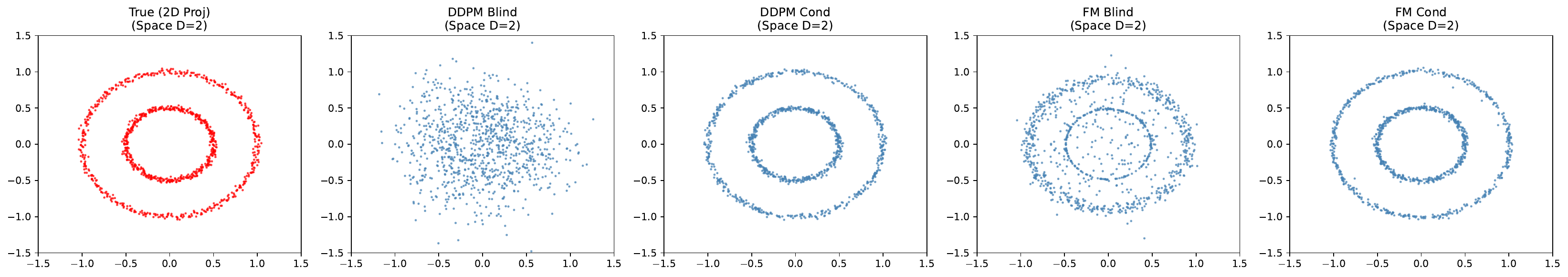}
\includegraphics[width=.95\textwidth]{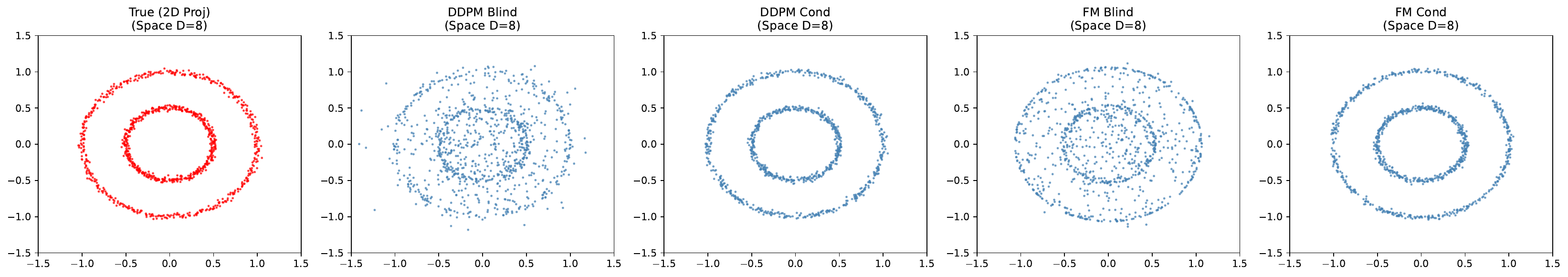}
\includegraphics[width=.95\textwidth]{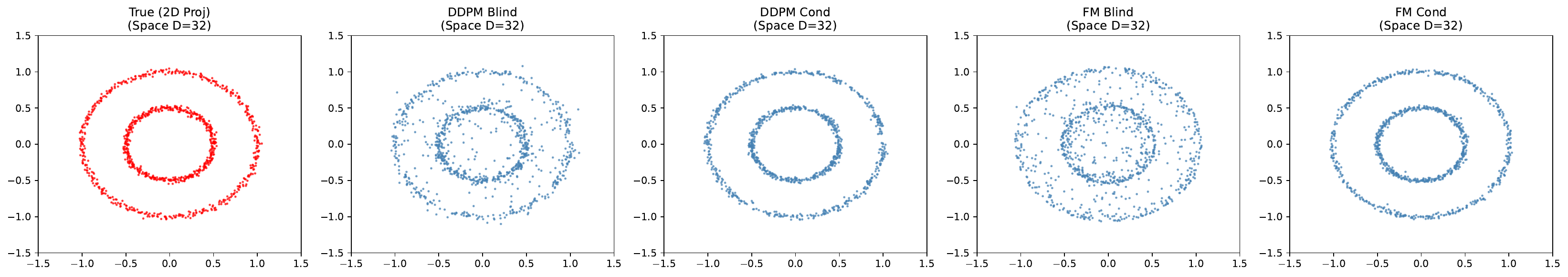}
\includegraphics[width=.95\textwidth]{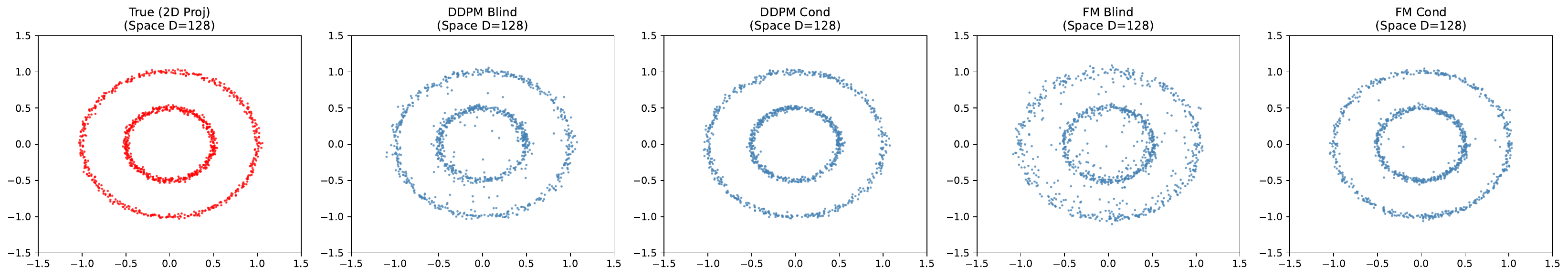}
\label{fig:toy-circles}
\caption{\textbf{Generative performance on a 2D concentric circles dataset embedded in $\mathbb{R}^D$.} Conditional models receive explicit time embeddings, while autonomous (blind) models must implicitly infer the noise scale. At low dimensions ($D=2$, top), blind models fail due to overlapping noise distributions. At moderate dimensions ($D=8, 32$, middle), Flow Matching achieves stable generation due to its bounded velocity parameterization, while DDPM Blind exhibits severe noise artifacts due to its singular gain $O(1/b(t))$. At extreme dimensions ($D=128$, bottom), absolute geometric concentration forces the estimation error to zero, allowing both blind models to converge.}
\end{figure}

\subsection{Experiments with more realistic datasets}
We verify our stability analysis against the benchmark results of \citet{sun2025isnoise}.

\paragraph{Quantitative Results.}
Table \ref{tab:cifar_results} confirms the theory on CIFAR-10. The failure of DDIM (FID 40.90) is not due to a lack of expressivity, but due to the structural instability of the parameterization. Velocity-based models (EqM, uEDM) achieve state-of-the-art performance by ensuring the learned field implicitly incorporates the Riemannian metric discussed in Section \ref{sec:riemannian_flow}.

\begin{table}[h]
\centering
\caption{Generative performance on CIFAR-10 reported by \citet{sun2025isnoise}. Stability correlates perfectly with bounded parameterization.}
\label{tab:cifar_results}
\begin{tabular}{llcc}
\toprule
\textbf{Model} & \textbf{Parameterization} & \textbf{Singularity} & \textbf{FID} (w/o $t$) \\
\midrule
DDIM \citep{song2020score} & Noise ($\epsilon$) & $O(1/b(t))$ & 40.90 \\
Flow Matching \citep{lipman2023flow} & Velocity ($v$) & Bounded & 2.61 \\
\textbf{uEDM \citep{sun2025isnoise}} & \textbf{Velocity ($v$)} & \textbf{Bounded} & \textbf{2.23} \\
\bottomrule
\end{tabular}
\end{table}

\section{Conclusion}

We have identified the marginal energy as the implicit objective of autonomous generative models and proved that its landscape contains a fundamental gradient singularity at the data manifold. We demonstrated that these models effectively implement a Riemannian gradient flow, where the posterior noise variance acts as a local conformal metric that preconditions the singular energy. Finally, we derived the bounded vector field condition, proving that velocity-based parameterizations are mathematically necessary to realize this stable flow in the absence of explicit noise conditioning. By shifting the generative task from time-dependent score matching to time-invariant energy alignment, our work provides a rigorous geometric foundation for the next generation of autonomous and equilibrium-based models.

\section*{Acknowledgments}
The authors would like to thank Ashwini Pokle and Sander Dieleman for helpful discussions.

\bibliography{ref}

@article{sun2025isnoise,
  title={Is Noise Conditioning Necessary for Denoising Generative Models?},
  author={Sun, Qiao and Jiang, Zhicheng and Zhao, Hanhong and He, Kaiming},
  journal={arXiv preprint arXiv:2502.13129},
  year={2025}
}

@article{wang2025equilibrium,
  title={Equilibrium Matching: Generative Modeling with Implicit Energy-Based Models},
  author={Wang, Runqian and Du, Yilun},
  journal={arXiv preprint arXiv:2510.02300},
  year={2025}
}

@inproceedings{ho2020denoising,
  title={Denoising Diffusion Probabilistic Models},
  author={Ho, Jonathan and Jain, Ajay and Abbeel, Pieter},
  booktitle={Advances in Neural Information Processing Systems},
  year={2020}
}

@inproceedings{song2020score,
  title={Score-Based Generative Modeling through Stochastic Differential Equations},
  author={Song, Yang and Sohl-Dickstein, Jascha and Kingma, Diederik P and Kumar, Abhishek and Ermon, Stefano and Poole, Ben},
  booktitle={International Conference on Learning Representations},
  year={2021}
}

@inproceedings{karras2022elucidating,
  title={Elucidating the Design Space of Diffusion-Based Generative Models},
  author={Karras, Tero and Aittala, Miika and Aila, Timo and Laine, Samuli},
  booktitle={Advances in Neural Information Processing Systems},
  year={2022}
}

@article{lipman2023flow,
  title={Flow Matching for Generative Modeling},
  author={Lipman, Yaron and Chen, Ricky TQ and Ben-Hamu, Heli and Nickel, Maximilian and Le, Matt},
  journal={arXiv preprint arXiv:2210.02747},
  year={2023}
}

@inproceedings{sohl2015deep,
  title={Deep Unsupervised Learning using Nonequilibrium Thermodynamics},
  author={Sohl-Dickstein, Jascha and Weiss, Eric and Maheswaranathan, Niru and Ganguli, Surya},
  booktitle={International Conference on Machine Learning},
  year={2015}
}

@article{jordan1998variational,
  title={The variational formulation of the Fokker-Planck equation},
  author={Jordan, Richard and Kinderlehrer, David and Otto, Felix},
  journal={SIAM journal on mathematical analysis},
  volume={29},
  number={1},
  pages={1--17},
  year={1998}
}

@inproceedings{neklyudov2023action,
  title={Action Matching: Learning Stochastic Dynamics from Samples},
  author={Neklyudov, Kirill and Brekelmans, Rob and Severo, Daniel and Makhzani, Alireza},
  booktitle={International Conference on Machine Learning},
  year={2023}
}

@inproceedings{song2019generative,
  title={Generative Modeling by Estimating Gradients of the Data Distribution},
  author={Song, Yang and Ermon, Stefano},
  booktitle={Advances in Neural Information Processing Systems},
  year={2019}
}

@inproceedings{song2020improved,
  title={Improved Techniques for Training Score-Based Generative Models},
  author={Song, Yang and Ermon, Stefano},
  booktitle={Advances in Neural Information Processing Systems},
  year={2020}
}

@article{gnanasambandam2020onesize,
  title={One Size Fits All: Can We Train One Denoiser for All Noise Levels?},
  author={Gnanasambandam, Abhiram and Chan, Stanley H},
  journal={International Conference on Machine Learning},
  year={2020}
}

@article{scarvelis2025closed,
  title={Closed-Form Diffusion Models},
  author={Scarvelis, Christopher and S{\'a}ez de Oc{\'a}riz Borde, Haitz and Solomon, Justin},
  journal={Transactions on Machine Learning Research},
  year={2025}
}

@article{guth2025learning,
  title={Learning normalized image densities via dual score matching},
  author={Guth, Florentin and Kadkhodaie, Zahra and Simoncelli, Eero P},
  journal={Advances in Neural Information Processing Systems},
  year={2025}
}

@inproceedings{du2021improved,
  title={Improved Contrastive Divergence Training of Energy-Based Models},
  author={Du, Yilun and Li, Shuang and Tenenbaum, Joshua and Mordatch, Igor},
  booktitle={International Conference on Machine Learning},
  year={2021}
}

@book{vershynin2018high,
  title={High-dimensional probability: An introduction with applications in data science},
  author={Vershynin, Roman},
  volume={47},
  year={2018},
  publisher={Cambridge university press}
}

@misc{kadkhodaie2026blinddenoisingdiffusionmodels,
      title={Blind denoising diffusion models and the blessings of dimensionality}, 
      author={Zahra Kadkhodaie and Aram-Alexandre Pooladian and Sinho Chewi and Eero Simoncelli},
      year={2026},
      eprint={2602.09639},
      archivePrefix={arXiv},
      primaryClass={cs.LG},
      url={https://arxiv.org/abs/2602.09639}, 
}

@article{li2025back,
  title={Back to basics: Let denoising generative models denoise},
  author={Li, Tianhong and He, Kaiming},
  journal={arXiv preprint arXiv:2511.13720},
  year={2025}
}

@incollection{robbins1992empirical,
  title={An empirical Bayes approach to statistics},
  author={Robbins, Herbert E},
  booktitle={Breakthroughs in Statistics: Foundations and basic theory},
  pages={388--394},
  year={1992},
  publisher={Springer}
}

@article{efron2011tweedie,
  title={Tweedie’s formula and selection bias},
  author={Efron, Bradley},
  journal={Journal of the American Statistical Association},
  volume={106},
  number={496},
  pages={1602--1614},
  year={2011},
  publisher={Taylor \& Francis}
}

@article{vincent2011connection,
  title={A connection between score matching and denoising autoencoders},
  author={Vincent, Pascal},
  journal={Neural computation},
  volume={23},
  number={7},
  pages={1661--1674},
  year={2011},
  publisher={MIT Press}
}

@book{ambrosio2005gradient,
  title={Gradient flows: in metric spaces and in the space of probability measures},
  author={Ambrosio, Luigi and Gigli, Nicola and Savar{\'e}, Giuseppe},
  year={2005},
  publisher={Springer}
}

@article{amari1998natural,
  title={Natural gradient works efficiently in learning},
  author={Amari, Shun-Ichi},
  journal={Neural computation},
  volume={10},
  number={2},
  pages={251--276},
  year={1998},
  publisher={MIT Press}
}

@article{lecun2006tutorial,
  title={A tutorial on energy-based learning},
  author={LeCun, Yann and Chopra, Sumit and Hadsell, Raia and Ranzato, M and Huang, Fujie and others},
  journal={Predicting structured data},
  volume={1},
  number={0},
  year={2006}
}

@book{ledoux2001concentration,
  title={The concentration of measure phenomenon},
  author={Ledoux, Michel},
  number={89},
  year={2001},
  publisher={American Mathematical Soc.}
}

@article{du2019implicit,
  title={Implicit generation and modeling with energy based models},
  author={Du, Yilun and Mordatch, Igor},
  journal={Advances in neural information processing systems},
  volume={32},
  year={2019}
}

@inproceedings{nichol2021improved,
  title={Improved denoising diffusion probabilistic models},
  author={Nichol, Alexander Quinn and Dhariwal, Prafulla},
  booktitle={International conference on machine learning},
  pages={8162--8171},
  year={2021},
  organization={PMLR}
}

@article{dhariwal2021diffusion,
  title={Diffusion models beat gans on image synthesis},
  author={Dhariwal, Prafulla and Nichol, Alexander},
  journal={Advances in neural information processing systems},
  volume={34},
  pages={8780--8794},
  year={2021}
}

@article{liu2022flow,
  title={Flow straight and fast: Learning to generate and transfer data with rectified flow},
  author={Liu, Xingchao and Gong, Chengyue and Liu, Qiang},
  journal={arXiv preprint arXiv:2209.03003},
  year={2022}
}

@article{albergo2025stochastic,
  title={Stochastic interpolants: A unifying framework for flows and diffusions},
  author={Albergo, Michael and Boffi, Nicholas M and Vanden-Eijnden, Eric},
  journal={Journal of Machine Learning Research},
  volume={26},
  number={209},
  pages={1--80},
  year={2025}
}

@article{delbracio2023inversion,
  title={Inversion by direct iteration: An alternative to denoising diffusion for image restoration},
  author={Delbracio, Mauricio and Milanfar, Peyman},
  journal={arXiv preprint arXiv:2303.11435},
  year={2023}
}

@inproceedings{
gao2025diffusion,
title={Diffusion Models and Gaussian Flow Matching: Two Sides of the Same Coin},
author={Ruiqi Gao and Emiel Hoogeboom and Jonathan Heek and Valentin De Bortoli and Kevin Patrick Murphy and Tim Salimans},
booktitle={The Fourth Blogpost Track at ICLR 2025},
year={2025},
url={https://openreview.net/forum?id=C8Yyg9wy0s}
}

@article{dieleman2023perspectives,
  title={Perspectives on diffusion},
  author={Dieleman, Sander},
  journal={Sander Dieleman's Blog},
  year={2023},
  url={https://sander.ai/2023/07/20/perspectives.html},
  note={Published: July 20, 2023}
}
\newpage

\appendix
\part*{Appendices}

\section{General Derivations for Autonomous Models}
\subsection{Derivation of the Optimal Autonomous Target}\label{sec:optimal_autonomous_target}
\begin{lemma}[Optimal Autonomous Target]
\label{lemma:optimal_target}
Consider the loss functional $\mathcal{L}(f)$ defined in Eq \eqref{eq:mse_loss}. The unique global minimizer $f^*(\vu)$ is given by the expectation of the target conditioned on the noise level $t$, weighted by the posterior $p(t|\vu)$:
\begin{equation}
    f^*(\vu) = \E_{t|\vu} \left[ \E_{\vx, \ve | \vu, t} [c(t)\vx + d(t)\ve] \right].
\end{equation}
\end{lemma}

\begin{proof}
This is an application of the Law of Iterated Expectations.
\end{proof}

\begin{lemma}[Denoiser Formulation]
\label{lemma:denoiser_formulation}
The optimal autonomous target $f^*(\vu)$ can be expressed as an affine transformation of the optimal conditional denoiser $D_t^*(\vu) = \E[\vx | \vu, t]$:
\begin{equation}
    f^{*}(\vu) = \mathbb{E}_{t|\vu} \left[ \frac{d(t)}{b(t)}\vu + \left( c(t) - \frac{d(t)a(t)}{b(t)} \right) D_{t}^{*}(\vu) \right].
\end{equation}
\end{lemma}

\begin{proof}
Recall the unified forward process $\vu = a(t)\vx + b(t)\ve$. For a fixed observation $\vu$ and noise level $t$, the noise $\ve$ is deterministically related to the clean data $\vx$ by:
\begin{equation}
    \ve = \frac{\vu - a(t)\vx}{b(t)}.
\end{equation}
Substitute this into the inner expectation of Lemma~\ref{lemma:optimal_target}. Note that conditioned on $\vu$ and $t$, the terms $\vu$, $a(t)$, and $b(t)$ are constants, leaving $\vx$ as the only random variable:
\begin{align}
    \E_{\vx, \ve | \vu, t} [c(t)\vx + d(t)\ve] &= \E_{\vx | \vu, t} \left[ c(t)\vx + d(t)\left( \frac{\vu - a(t)\vx}{b(t)} \right) \right] \\
    &= c(t)\E[\vx | \vu, t] + \frac{d(t)}{b(t)}\vu - \frac{d(t)a(t)}{b(t)}\E[\vx | \vu, t].
\end{align}
Identifying $D_t^*(\vu) = \E[\vx | \vu, t]$ and grouping the coefficients for $\vu$ and $D_t^*(\vu)$ yields the result.
\end{proof}

\subsection{Gradient of the Marginal Energy}
\begin{lemma}[Gradient of the Marginal Energy]
\label{lemma:marginal_grad}
Let the marginal likelihood be the mixture $p(\vu) = \int p(\vu|t)p(t)dt$ and the marginal energy be $\Emarg(\vu) = -\log p(\vu)$. The gradient of the marginal energy is the posterior expectation of the conditional energy gradients:
\begin{equation}
    \nabla_\vu \Emarg(\vu) = \E_{t|\vu} \left[ -\nabla_\vu \log p(\vu|t) \right].
\end{equation}
\end{lemma}

\begin{proof}
By definition, $\nabla_\vu \Emarg(\vu) = -\frac{\nabla_\vu p(\vu)}{p(\vu)}$. We differentiate the mixture integral under the sign:
\begin{equation}
    \nabla_\vu p(\vu) = \nabla_\vu \int p(\vu|t)p(t) dt = \int \nabla_\vu p(\vu|t)p(t) dt.
\end{equation}
We use the log-derivative trick $\nabla_\vu p(\vu|t) = p(\vu|t) \nabla_\vu \log p(\vu|t)$ to rewrite the integrand:
\begin{equation}
    \nabla_\vu p(\vu) = \int p(\vu|t) \nabla_\vu \log p(\vu|t) p(t) dt.
\end{equation}
Dividing by $p(\vu)$ allows us to identify the posterior density $p(t|\vu) = \frac{p(\vu|t)p(t)}{p(\vu)}$:
\begin{align}
    \nabla_\vu \Emarg(\vu) &= - \int \frac{p(\vu|t)p(t)}{p(\vu)} \nabla_\vu \log p(\vu|t) dt \\
    &= \int p(t|\vu) \left[ -\nabla_\vu \log p(\vu|t) \right] dt.
\end{align}
\end{proof}

\subsection{Exact Analytical Forms for Autonomous Fields}
\label{sec:exact_analytical_forms}

To facilitate the exact stability verification in Section~\ref{sec:stability} and Appendix~\ref{sec:stability_details}, we derive the closed-form expressions for the optimal autonomous vector fields. While the neural network approximates these expectations, we can compute them exactly when the data distribution $p_{data}(\vx)$ is known.

\paragraph{General Formulation.}
Recall from Lemma~\ref{lemma:denoiser_formulation} that the optimal autonomous field is an affine transformation of the posterior expectation of the conditional denoiser:
\begin{equation}
    f^*(\vu) = \mathbb{E}_{t|\vu} \left[ \frac{d(t)}{b(t)}\vu + \left(c(t) - \frac{d(t)a(t)}{b(t)}\right) D_t^*(\vu) \right].
\end{equation}
The two key components required to evaluate this are:
\begin{enumerate}
    \item \textbf{The Optimal Conditional Denoiser} $D_t^*(\vu) = \mathbb{E}[\vx|\vu, t]$. By Bayes' rule, this is the center of mass of the posterior $p(\vx|\vu, t) \propto p(\vu|\vx, t)p_{data}(\vx)$.
    \item \textbf{The Posterior Noise Distribution} $p(t|\vu)$. This allows us to average the conditional vector field over all possible noise levels.
\end{enumerate}

\paragraph{Specialization to Discrete Data.}
Let the data manifold be a discrete set $\mathcal{X} = \{\vx_k\}_{k=1}^N$ with uniform prior $p(\vx_k) = 1/N$.

1. \textbf{Conditional Denoiser:} The likelihood of observing $\vu$ given a specific source $\vx_k$ is Gaussian: $p(\vu | \vx_k, t) = \mathcal{N}(\vu; a(t)\vx_k, b(t)^2 I)$. The posterior probability $w_k(\vu, t) \triangleq p(\vx_k | \vu, t)$ is given by the softmax of the negative log-likelihoods:
\begin{equation}
    w_k(\vu, t) = \frac{\exp\left(-\frac{\|\vu - a(t)\vx_k\|^2}{2b(t)^2}\right)}{\sum_{j=1}^N \exp\left(-\frac{\|\vu - a(t)\vx_j\|^2}{2b(t)^2}\right)}.
\end{equation}
The optimal denoiser is the precision-weighted barycenter of the dataset:
\begin{equation}
    D_t^*(\vu) = \sum_{k=1}^N w_k(\vu, t) \vx_k.
    \label{eq:exact_denoiser}
\end{equation}

2. \textbf{Posterior Noise Distribution:} To compute the outer expectation $\mathbb{E}_{t|\vu}[\cdot]$, we require $p(t|\vu)$. The marginal likelihood $p(\vu|t)$ is a Gaussian Mixture Model (GMM) with centers at $a(t)\vx_k$. Assuming a uniform prior on time $p(t) \sim \mathcal{U}(0,1)$, Bayes' rule yields:
\begin{equation}
    p(t|\vu) = \frac{p(\vu|t)}{\int_0^1 p(\vu|\tau) d\tau} \propto \frac{1}{N} \sum_{k=1}^N \mathcal{N}(\vu; a(t)\vx_k, b(t)^2 I).
\end{equation}
Therefore, for the case of discrete data prior, we can evaluate the unconditional field $f^*(\vu)$ by numerically integrating Eq.~\eqref{eq:exact_denoiser} against this posterior $p(t|\vu)$ using high-precision quadrature on a dense grid $t \in [\epsilon, 1]$. This eliminates sampling variance and allows us to isolate the geometric stability properties of the parameterization.

\section{Concentration of $p(t|\vu)$ via Proximity}
\label{sec:proof_concentration}

In this section, we provide a rigorous proof that for autonomous generative models, the posterior distribution $p(t|\vu)$ converges weakly to a Dirac measure $\delta(t)$ as the observation $\vu$ approaches the data support $\mathcal{X}$.

We prove this result for two separate cases. In Case 1, we assume that the data is discrete and finite. In this case, we prove that if the ambient dimension satisfies $D > 2$, regardless of the number of data points, as we approach a data point the posterior $p(t|\vu)$ converges weakly to the Dirac measure $\delta(t)$. In the second case, we assume that the data lies on a manifold, but the dimension of the manifold $d$ and the ambient dimension $D$ satisfy $D-d > 2$. In this case, we show that as we approach the data manifold, the posterior $p(t|\vu)$ again converges weakly to the Dirac measure $\delta(t)$.

We first establish a general lemma for the concentration of the specific family of distributions that arise in our analysis.

\begin{lemma}[Concentration of the Inverse-Gamma Kernel]
\label{lemma:ig_concentration}
Let $q(v; \beta) \propto v^{-\alpha} \exp(-\frac{\beta}{v})$ be a probability density on $v > 0$, with fixed shape $\alpha > 1$ and scale parameter $\beta > 0$. As the scale $\beta \to 0$, the distribution $q(v; \beta)$ converges weakly to a Dirac mass at zero: $q(v; \beta) \xrightarrow{w} \delta(v)$.
\end{lemma}

\begin{proof}
This density is an Inverse-Gamma distribution $\mathcal{IG}(\alpha-1, \beta)$. The mean and variance are given by $\mathbb{E}[v] = \frac{\beta}{\alpha - 2}$ and $\text{Var}(v) = \frac{\beta^2}{(\alpha - 2)^2 (\alpha - 3)}$ (for $\alpha > 3$).
As $\beta \to 0$, both the mean and the variance vanish. By Chebyshev's inequality, the random variable $v$ converges in probability to 0, which implies weak convergence to $\delta(v)$.
\end{proof}

\subsection{Case 1: Discrete Data Support}

\begin{lemma}
Let the data support be a discrete set $\mathcal{X} = \{\vx_k\}_{k=1}^N$. Let $\vu = \vx_j + \boldsymbol{\delta}$ with $\|\boldsymbol{\delta}\| = \epsilon$. Assume the ambient dimension $D > 2$, the prior $p(t)$ is continuous with $p(0) > 0$, and the noise schedule $b(t)$ is continuous, strictly increasing, with $b(0)=0$. Then as $\epsilon \to 0$, $p(t|\vu) \xrightarrow{w} \delta(t)$.
\end{lemma}

\begin{proof}
Let $v = b(t)^2$ be the variance. Since $b(t)$ is strictly increasing and continuous with $b(0)=0$, the mapping $t \mapsto v$ is a homeomorphism near zero. Therefore, proving $p(v|\vu) \xrightarrow{w} \delta(v)$ is sufficient to imply $p(t|\vu) \xrightarrow{w} \delta(t)$.

The marginal likelihood of the observation is a mixture of Gaussians:
\begin{equation}
    p(\vu|v) = \frac{1}{N} \sum_{k=1}^N \underbrace{(2\pi v)^{-D/2} \exp\left( -\frac{\|\vu - \vx_k\|^2}{2v} \right)}_{L_k(v)}.
\end{equation}
By Bayes' rule, the posterior $p(v|\vu)$ is a mixture of the individual component posteriors:
\begin{equation}
    p(v|\vu) = \sum_{k=1}^N W_k(\epsilon) p_k(v|\vu),
\end{equation}
where $p_k(v|\vu) = \frac{L_k(v)p(v)}{Z_k(\epsilon)}$ is the normalized component posterior, and $Z_k(\epsilon) = \int_0^\infty L_k(z)p(z)dz$ is the component evidence. The mixing weights are determined by the ratio of component evidences: $W_k(\epsilon) = \frac{Z_k(\epsilon)}{\sum_{i=1}^N Z_i(\epsilon)}$. Note that since $p(t)$ is bounded below near zero, the induced prior $p(v)$ satisfies $p(v) \geq c_0 > 0$ in a neighborhood of $v=0$.

We analyze the asymptotic behavior of the evidence integrals as $\epsilon \to 0$. For the nearest neighbor ($k=j$), the squared distance is $\|\vu - \vx_j\|^2 = \epsilon^2$. The evidence integral is:
\begin{equation}
    Z_j(\epsilon) = \int_0^\infty (2\pi z)^{-D/2} \exp\left(-\frac{\epsilon^2}{2z}\right) p(z) dz.
\end{equation}
Substituting $y = \frac{\epsilon^2}{2z}$ yields $Z_j(\epsilon) \propto (\epsilon^2)^{1 - D/2} \int_0^\infty y^{D/2 - 2} e^{-y} p(\frac{\epsilon^2}{2y}) dy$. For $D > 2$, the integral converges to a finite value bounded away from zero. Consequently, the evidence scales as $Z_j(\epsilon) = \mathcal{O}(\epsilon^{2-D})$. Since $2-D < 0$, $Z_j(\epsilon) \to \infty$ as $\epsilon \to 0$.

In contrast, for any other point $k \neq j$, the distance $\|\vu - \vx_k\|^2 \to \Delta_{jk}^2 > 0$. The evidence integral $Z_k(\epsilon)$ converges to a finite constant $Z_k(0) < \infty$, as the exponential term $\exp(-\Delta_{jk}^2/2z)$ suppresses the singularity at $z=0$.

This divergence of the nearest-neighbor evidence implies that the mixing weights collapse:
\begin{equation}
    W_j(\epsilon) = \frac{1}{1 + \sum_{k \neq j} \frac{Z_k(\epsilon)}{Z_j(\epsilon)}} \to 1.
\end{equation}
Thus, the posterior is asymptotically dominated by the nearest neighbor component $p_j(v|\vu) \propto v^{-D/2} \exp(-\frac{\epsilon^2}{2v}) p(v)$. This is an Inverse-Gamma kernel with scale parameter $\beta = \epsilon^2/2$. As $\epsilon \to 0$, the scale vanishes, and by Lemma \ref{lemma:ig_concentration}, $p(v|\vu) \xrightarrow{w} \delta(v)$.
\end{proof}

\subsection{Case 2: Continuous Low-Dimensional Manifold}

\begin{lemma}
Let data lie on a smooth ($C^2$) $d$-dimensional submanifold $\mathcal{M} \subset \mathbb{R}^D$ with codimension $k = D-d > 2$. Let $r = \text{dist}(\vu, \mathcal{M})$ be the orthogonal distance. Assume $p_{data}$ is continuous and bounded on $\mathcal{M}$. As $r \to 0$, $p(t|\vu) \xrightarrow{w} \delta(t)$.
\end{lemma}

\begin{proof}
We analyze the marginal likelihood integral over the manifold $\mathcal{M}$ as a function of the variance $v=b(t)^2$:
\begin{equation}
    p(\vu|v) = \int_{\mathcal{M}} (2\pi v)^{-D/2} \exp\left( -\frac{\|\vu - \vx\|^2}{2v} \right) p_{data}(\vx) d\vx.
\end{equation}
Let $\vx_{proj}$ be the orthogonal projection of the fixed observation $\vu$ onto $\mathcal{M}$. Since $\vu$ is fixed, $\vx_{proj}$ is a constant vector.
We define local coordinates on the manifold centered at $\vx_{proj}$. Let $\mathbf{y} \in \mathbb{R}^d$ represent coordinates in the tangent space $T_{\vx_{proj}}\mathcal{M}$. For points $\vx$ near $\vx_{proj}$, we have the expansion $\vx(\mathbf{y}) = \vx_{proj} + \mathbf{J}\mathbf{y} + O(\|\mathbf{y}\|^2)$.
The squared distance decomposes as:
\begin{equation}
    \|\vu - \vx\|^2 = \|\vu - \vx_{proj} + \vx_{proj} - \vx\|^2 = r^2 + \|\vx_{proj} - \vx\|^2 \approx r^2 + \|\mathbf{y}\|^2.
\end{equation}
(The cross-term vanishes because $\vu - \vx_{proj}$ is orthogonal to the tangent space).

Substituting this into the integral:
\begin{equation}
    p(\vu|v) \approx (2\pi v)^{-D/2} e^{-\frac{r^2}{2v}} \int_{\mathbb{R}^d} e^{-\frac{\|\mathbf{y}\|^2}{2v}} p_{data}(\vx(\mathbf{y})) d\mathbf{y}.
\end{equation}
We apply the Laplace method for asymptotic integrals as $v \to 0$ (which corresponds to the high-likelihood regime). The Gaussian kernel $e^{-\|\mathbf{y}\|^2/2v}$ concentrates mass entirely at $\mathbf{y}=0$ (i.e., $\vx = \vx_{proj}$). Since $p_{data}$ is continuous, we can pull the value $p_{data}(\vx_{proj})$ out of the integral:
\begin{equation}
    \int_{\mathbb{R}^d} e^{-\frac{\|\mathbf{y}\|^2}{2v}} p_{data}(\vx(\mathbf{y})) d\mathbf{y} \approx p_{data}(\vx_{proj}) \int_{\mathbb{R}^d} e^{-\frac{\|\mathbf{y}\|^2}{2v}} d\mathbf{y}.
\end{equation}
The remaining integral is a standard unnormalized Gaussian integral over $d$ dimensions, equal to $(2\pi v)^{d/2}$.
Combining the pre-factors:
\begin{equation}
    p(\vu|v) \propto (2\pi v)^{-D/2} \cdot (2\pi v)^{d/2} \cdot e^{-\frac{r^2}{2v}} = (2\pi v)^{-(D-d)/2} \exp\left( -\frac{r^2}{2v} \right).
\end{equation}
Let $k = D-d$ be the codimension. The likelihood takes the form of an Inverse-Gamma kernel:
\begin{equation}
    p(\vu|v) \propto v^{-k/2} \exp\left( -\frac{r^2}{2v} \right).
\end{equation}
Assuming a flat prior on $v$ near 0 (or bounded $p(v)$), the posterior $p(v|\vu)$ is an Inverse-Gamma distribution $\mathcal{IG}(\alpha, \beta)$ with shape $\alpha = \frac{k}{2} - 1$ and scale $\beta = \frac{r^2}{2}$.
Provided $k > 2$, the shape parameter $\alpha > 0$. As the distance to the manifold $r \to 0$, the scale parameter $\beta \to 0$. By Lemma \ref{lemma:ig_concentration}, the distribution converges weakly to a Dirac mass: $p(v|\vu) \xrightarrow{w} \delta(v)$. This implies $p(t|\vu) \xrightarrow{w} \delta(t)$.
\end{proof}

\section{Posterior Concentration in High Dimensions}
\label{app:concentration_proof}

In this section, we prove that in high dimensions, the posterior distribution of the noise level concentrates sharply, effectively allowing noise-blind models to recover the time signal from the spatial geometry of the observation $\vu$.

\begin{proposition}
Let $\mathbb{V} \subset \mathbb{R}^D$ be a linear subspace of dimension $d < D$. Let $\vu = \vx + \mathbf{n}$ be an observation where $\vx \in \mathbb{V}$ and $\mathbf{n} \sim \mathcal{N}(0, \sigma^2 I_D)$. Assuming an improper flat prior $p(\vx) \propto 1$ on $\mathbb{V}$, the posterior distribution $p(\sigma|\vu)$ concentrates at $\hat{\sigma} = \frac{r}{\sqrt{D-d}}$ with variance $\mathcal{O}(D^{-1})$, where $r = \min_{\mathbf{y} \in \mathbb{V}} \|\vu - \mathbf{y}\|$.
\end{proposition}

\begin{proof}
Let $\vx^* = \text{proj}_{\mathbb{V}}(\vu)$ be the orthogonal projection of $\vu$ onto $\mathbb{V}$. For any $\vx \in \mathbb{V}$, the Pythagorean theorem yields the exact decomposition:
\begin{equation*}
\|\vu - \vx\|^2 = \|\vu - \vx^*\|^2 + \|\vx^* - \vx\|^2 = r^2 + \|\vx^* - \vx\|^2
\end{equation*}

The marginal likelihood (or integrated likelihood) is defined as:
\begin{align*}
p(\vu|\sigma) &= \int_{\mathbb{V}} p(\vu|\vx, \sigma)p(\vx) \, d\vx \\
&= \frac{1}{(2\pi\sigma^2)^{D/2}} \int_{\mathbb{V}} \exp\left( -\frac{r^2 + \|\vx^* - \vx\|^2}{2\sigma^2} \right) \, d\vx
\end{align*}
By defining an isometric isomorphism between $\mathbb{V}$ and $\mathbb{R}^d$, we evaluate the integral over the subspace:
\begin{align*}
p(\vu|\sigma) &= \frac{\exp(-r^2/2\sigma^2)}{(2\pi\sigma^2)^{D/2}} \int_{\mathbb{R}^d} \exp\left( -\frac{\|\mathbf{z}\|^2}{2\sigma^2} \right) \, d\mathbf{z} \\
&= \frac{\exp(-r^2/2\sigma^2)}{(2\pi\sigma^2)^{D/2}} (2\pi\sigma^2)^{d/2} \\
&= (2\pi\sigma^2)^{-\frac{D-d}{2}} \exp\left( -\frac{r^2}{2\sigma^2} \right)
\end{align*}

The log-marginal likelihood $\ell(\sigma) = \ln p(\vu|\sigma)$ is:
\begin{equation*}
\ell(\sigma) = -(D - d) \ln \sigma - \frac{r^2}{2\sigma^2} + \text{const}
\end{equation*}
Solving $\ell'(\sigma) = 0$ gives the maximum likelihood estimator:
\begin{equation*}
-\frac{D - d}{\sigma} + \frac{r^2}{\sigma^3} = 0 \implies \hat{\sigma}^2 = \frac{r^2}{D - d}
\end{equation*}

The observed Fisher Information $\mathcal{I}(\hat{\sigma}) = -\ell''(\hat{\sigma})$ evaluated at $\hat{\sigma}$ is:
\begin{equation*}
\mathcal{I}(\hat{\sigma}) = \left[ -\frac{D - d}{\sigma^2} + \frac{3r^2}{\sigma^4} \right]_{\sigma=\hat{\sigma}} = \frac{2(D - d)}{\hat{\sigma}^2}
\end{equation*}
As $D \to \infty$ (with $d$ fixed), the posterior variance $\text{Var}(\sigma|\vu) \approx \mathcal{I}(\hat{\sigma})^{-1} \propto (D-d)^{-1}$ vanishes via the Laplace approximation. Furthermore, since $r^2/\sigma_0^2 \sim \chi^2_{D-d}$, the estimator $\hat{\sigma}^2$ satisfies $\mathbb{E}[\hat{\sigma}^2] = \sigma_0^2$ and $\text{Var}(\hat{\sigma}^2) = \frac{2\sigma_0^4}{D-d} \to 0$. By the Continuous Mapping Theorem, $\hat{\sigma}$ is a consistent estimator of the true noise level $\sigma_0$, and the posterior distribution $p(\sigma|\vu)$ concentrates at the true value $\sigma_0$.
\end{proof}

\section{Derivation of General Energy-Aligned Decomposition}\label{app:covariance_decomposition}

In this section, we derive the exact relationship between the learned autonomous vector field $f^{*}(\vu)$ and the gradient of the marginal energy $\nabla \Emarg(\vu)$ for general affine noise schedules.

Recall the general autonomous form derived in Eq. (\ref{eq:autonomous_target_expansion}):
\begin{equation}
    f^{*}(\vu) = \mathbb{E}_{t|\vu} \left[ \frac{d(t)}{b(t)}\vu + \left( c(t) - \frac{d(t)a(t)}{b(t)} \right) D_{t}^{*}(\vu) \right]
\end{equation}
We defined the conditional energy as
\begin{equation}
    E_{t}(\vu) = -\log p(\vu|t).
\end{equation}
To relate this to the energy landscape, we substitute the Generalized Tweedie's formula (Eq. 4), which expresses the optimal denoiser in terms of the conditional energy gradient $\nabla E_{t}(\vu)$:
\begin{equation}
    D_{t}^{*}(\vu) = \frac{\vu - b(t)^{2}\nabla E_{t}(\vu)}{a(t)}
\end{equation}
Substituting this into the expression for $f^{*}(\vu)$, we obtain:
\begin{equation}
    f^{*}(\vu) = \mathbb{E}_{t|\vu} \left[ \frac{d(t)}{b(t)}\vu + \left( c(t) - \frac{d(t)a(t)}{b(t)} \right) \left( \frac{\vu - b(t)^{2}\nabla E_{t}(\vu)}{a(t)} \right) \right]
\end{equation}
We simplify the coefficients for the linear term $\vu$ and the gradient term $\nabla E_{t}(\vu)$ separately.

\textbf{1. The Linear Term:}
The coefficient for $\vu$ is:
\begin{equation}
    \frac{d(t)}{b(t)} + \frac{1}{a(t)}\left( c(t) - \frac{d(t)a(t)}{b(t)} \right) = \frac{d(t)}{b(t)} + \frac{c(t)}{a(t)} - \frac{d(t)}{b(t)} = \frac{c(t)}{a(t)}
\end{equation}
Thus, the linear component of the field is simply the posterior expectation of the target scale ratio:
\begin{equation}
    f^{*}_{linear}(\vu) = \vu \cdot \mathbb{E}_{t|\vu} \left[ \frac{c(t)}{a(t)} \right]
\end{equation}

\textbf{2. The Gradient Term:}
The coefficient for $\nabla E_{t}(\vu)$ is:
\begin{equation}
    -\frac{b(t)^{2}}{a(t)} \left( c(t) - \frac{d(t)a(t)}{b(t)} \right) = -\frac{b(t)^{2}c(t)}{a(t)} + b(t)d(t) = \frac{b(t)}{a(t)} \left( d(t)a(t) - c(t)b(t) \right)
\end{equation}
Let us define this effective gradient gain as $\lambda(t)$:
\begin{equation}
    \lambda(t) \triangleq \frac{b(t)}{a(t)} \left( d(t)a(t) - c(t)b(t) \right)
\end{equation}
The autonomous field can now be written as:
\begin{equation}
    f^{*}(\vu) = \mathbb{E}_{t|\vu} \left[ \lambda(t) \nabla E_{t}(\vu) \right] + \vu \cdot \mathbb{E}_{t|\vu} \left[ \frac{c(t)}{a(t)} \right]
\end{equation}
Finally, we apply the covariance decomposition to the expectation term. Recall that the marginal energy gradient is the average of the conditional gradients: $\nabla \Emarg(\vu) = \mathbb{E}_{t|\vu}[\nabla E_{t}(\vu)]$. Using the identity $\mathbb{E}[XY] = \mathbb{E}[X]\mathbb{E}[Y] + Cov(X,Y)$, we derive the \textbf{General Energy-Aligned Decomposition}:

\begin{equation}
    f^{*}(\vu) = \underbrace{\overline{\lambda}(u)\nabla \Emarg(\vu)}_{\text{Natural Gradient}} + \underbrace{Cov(\lambda(t), \nabla E_{t}(\vu))}_{\text{Transport Correction}} + \underbrace{\overline{c}_{scale}(\vu)\vu}_{\text{Linear Drift}}
\end{equation}

where $\overline{\lambda}(\vu) = \mathbb{E}_{t|\vu}[\lambda(t)]$ is the posterior effective gain, and $\overline{c}_{scale}(\vu) = \mathbb{E}_{t|\vu}[c(t)/a(t)]$ is the mean linear drift coefficient. This result proves that for any affine schedule, the learned field is a sum of the marginal energy gradient (scaled by posterior uncertainty), a covariance correction term, and a linear drift.

\section{Analysis of Specific Architectures: Exact Low-Noise Asymptotics}\label{sec:asymptotics}

In this section, we show that near the data manifold, the effective gain $\overline{\lambda}(\vu)$ in the vector field vanishes at a rate that perfectly counteracts the divergence of the gradient of the marginal energy. We assume discrete data to analyze the asymptotic behavior of the learned autonomous field $f^*(\vu)$ near the data manifold ($\vu \to \mathcal{X}$). We derive the exact form using the limit of the marginal energy gradient:
\begin{equation}
    \nabla_{\vu} E_{marg}(\vu) \approx \frac{\vu - a(t)\mathbf{x}}{b(t)^2}
    \label{eq:grad_asymp}
\end{equation}
Substituting this into the General Energy-Aligned Decomposition ($f^* = \overline{\lambda}\nabla E + \text{Drift}$), we examine the properties of the target learned by each parameterization.

\textbf{1. DDPM and EDM (Noise/Data Prediction)}
For these models, $a(t) \approx 1$. The effective gain is $\lambda(t) \approx b(t)$ and the drift is zero.
\begin{equation}
    f^{*}(\vu) \approx \underbrace{b(t)}_{\text{Gain}} \cdot \underbrace{\frac{\vu - \mathbf{x}}{b(t)^2}}_{\text{Gradient}} = \frac{\vu - \mathbf{x}}{b(t)} \sim O(1)
\end{equation}
\textbf{Result (Bounded Target):} The learned target (which corresponds to $\epsilon$ or scaled data) contains a removable singularity of order $O(1)$. While the target itself is bounded, it does not inherently define the flow dynamics. Standard diffusion ODEs typically scale this target by $1/b(t)$ (i.e., $d\vu/dt \propto f^*/b(t)$). Thus, while the network learns a stable quantity, the implied autonomous dynamics may still diverge without careful parameterization (see Section 6).

\textbf{2. Flow Matching (Velocity Prediction)}
For Flow Matching, $a(t) \approx 1-t$ and $b(t) \approx t$. The gain is $\lambda(t) \approx t$ and drift is $\approx -\mathbf{x}$.
\begin{equation}
    f^{*}(\vu) \approx \underbrace{t}_{\text{Gain}} \cdot \underbrace{\frac{\vu - (1-t)\mathbf{x}}{t^2}}_{\text{Gradient}} \underbrace{- \mathbf{x}}_{\text{Drift}} = \frac{\vu - \mathbf{x}}{t} \sim O(1)
\end{equation}
\textbf{Result (Stable Transport):} The learned target simplifies to a finite velocity vector. Crucially, for Flow Matching, this target \textit{is} the ODE velocity ($d\vu/dt = f^*(\vu)$). Since the target is $O(1)$, the resulting generation trajectories approach the manifold with finite speed, ensuring stable transport without numerical explosion.

\textbf{3. Equilibrium Matching (Stabilized Transport)}
For EqM, the gain is higher-order $\lambda(t) \approx t^2$, and the drift coefficient vanishes.
\begin{equation}
    f^{*}(\vu) \approx \underbrace{t^2}_{\text{Gain}} \cdot \underbrace{\frac{\vu - (1-t)\mathbf{x}}{t^2}}_{\text{Gradient}} + 0 = \vu - (1-t)\mathbf{x} \xrightarrow{t \to 0} \vu - \mathbf{x}
\end{equation}
\textbf{Result (Vanishing Equilibrium):} The learned target scales with the distance to the manifold, vanishing as $O(\|\vu-\mathbf{x}\|)$. Since EqM uses this target directly as the ODE velocity, the dynamics naturally slow down and stop at the data ($d\vu/dt \to 0$). This creates a stable fixed point at the manifold, contrasting with the constant-velocity transport of Flow Matching.

\section{Derivation of Stability Conditions}
\label{sec:stability_details}

In this appendix, we provide the derivation of the Unified Sampler Dynamics and the rigorous proofs for the stability limits of the three parameterizations discussed in Section~\ref{sec:stability}.

\subsection{Unified Sampler Dynamics}

For general affine noise schedules defined by $\vu_t = a(t)\vx + b(t)\ve$, the generation process involves integrating a differential equation. We derive the exact ODE by inverting the linear system relating the data, noise, and observation.

The flow of the process is given by differentiating the forward process:
\begin{equation}
    \dot{\vu} = \dot{a}(t)\vx + \dot{b}(t)\ve.
\end{equation}
An autonomous generative model predicts a target $f^*(\vu) = c(t)\vx + d(t)\ve$. Combining this with the observation identity $\vu = a(t)\vx + b(t)\ve$, we form the linear system:
\begin{equation}
    \begin{bmatrix} \vu \\ f^*(\vu) \end{bmatrix} = \begin{bmatrix} a(t) & b(t) \\ c(t) & d(t) \end{bmatrix} \begin{bmatrix} \vx \\ \ve \end{bmatrix}.
\end{equation}
Solving for $\vx$ and $\ve$ and substituting them into the flow equation $\dot{\vu}$, we obtain the general sampler ODE:
\begin{equation}
    \frac{d\vu}{dt} = \underbrace{\left( \frac{\dot{a}d - \dot{b}c}{ad - bc} \right)}_{\mu(t)} \vu + \underbrace{\left( \frac{\dot{b}a - \dot{a}b}{ad - bc} \right)}_{\nu(t)} f^*(\vu).
    \label{eq:appendix_sampler_ode}
\end{equation}
We identify $\mu(t)$ as the schedule drift coefficient and $\nu(t)$ as the effective gain of the parameterization.

In order to analyze the stability of generation, we also consider generation using the optimal conditional model $f^*_t(\vu)$ that has access to the noise level. We define the Drift Perturbation Error $\Delta \vv$ as the norm difference between the autonomous drift (using $f^*(\vu)$) and the oracle drift. Since the linear term $\mu(t)\vu$ is identical for both, it cancels out giving us:
\begin{equation}
    \Delta \vv(\vu, t) = |\nu(t)| \cdot \| f^*(\vu) - f_t^*(\vu) \|.
\end{equation}
If this Drift Perturbation Error has singularities in it, this results in unstable generation dynamics.

\subsection{Stability Analysis by Parameterization}

We evaluate the limit of $\Delta \vv$ as $t \to 0$ (near the data manifold) for standard models. We assume standard boundary conditions $a(t) \to 1$ and $b(t) \to 0$.

\paragraph{Case 1: Noise Prediction (DDPM/DDIM)}
Target: $\ve$ ($c=0, d=1$).
The effective gain simplifies to $\nu(t) = \frac{\dot{b}a - \dot{a}b}{a}$. Near the manifold ($a \approx 1, \dot{a} \approx 0$), the gain behaves as the noise derivative: $\nu(t) \approx \dot{b}(t)$.
The optimal conditional target near the manifold is given by the geometric relation $\ve^*_t(\vu) \approx \frac{\vu - \vx}{b(t)}$.
Substituting this into the error norm:
\begin{equation}
    \Delta \vv_{noise} \approx \dot{b}(t) \left\| \mathbb{E}_{\tau|\vu}\left[\frac{\vu - \vx}{b(\tau)}\right] - \frac{\vu - \vx}{b(t)} \right\|.
\end{equation}
Factoring out the geometric direction $\|\vu - \vx\|$, we isolate the scaling behavior:
\begin{equation}
    \Delta \vv_{noise} \approx \|\vu - \vx\| \left| \frac{\dot{b}(t)}{b(t)} \underbrace{\left( b(t)\mathbb{E}_{\tau|\vu}\left[\frac{1}{b(\tau)}\right] - 1 \right)}_{\text{Jensen Gap}} \right|.\label{eq:amplification_factor}
\end{equation}
We call the term in the parenthesis the ``Jensen Gap'', the difference between the harmonic mean of noise levels and the true noise level, which converges to a non-zero constant due to the strict convexity of $1/x$ unless the posterior $p({\tau|\vu})$ converges to a Dirac measure $\delta(\tau - t)$. The instability is driven by the pre-factor $\frac{\dot{b}(t)}{b(t)} = \frac{d}{dt} \ln b(t)$. For any polynomial noise schedule $b(t) \propto t^k$ (where $k > 0$), this term diverges as $O(1/t)$.
Consequently, $\lim_{t \to 0} \Delta \vv_{noise} = \infty$, rendering the dynamics structurally unstable for autonomous noise prediction. Note that for Variance Preserving SDEs where $b(t) \propto \sqrt{t}$, the singularity is even stronger ($O(t^{-1.5})$), but the divergence exists even for linear schedules ($b(t)=t$) that is used in flow matching.

\paragraph{Case 2: Signal Prediction (EDM)}
Target: $\vx$ ($c=1, d=0$).
The effective gain scales as $\nu(t) \approx \frac{1}{b(t)^2}$.
The drift error is determined by the denoising error:
\begin{equation}
    \Delta \vv_{signal} \approx \frac{1}{b(t)^2} \| \hat{\vx}(\vu) - \vx_t^*(\vu) \|.
\end{equation}
To resolve the $0/0$ indeterminacy, we assume the data manifold is discrete, $\mathcal{X}=\{\vx_k\}_{k=1}^N$. Near a specific data point $\vx_k$, the optimal conditional denoiser $\vx_t^*(\vu)$ is a softmax-weighted average of the dataset. The error is dominated by the distance to the nearest neighbor $\vx_j$, with a weight proportional to the Gaussian likelihood ratio:
\begin{equation}
    \| \vx_t^*(\vu) - \vx_k \| \propto \exp\left( -\frac{\|\vx_j - \vx_k\|^2}{2b(t)^2} \right).
\end{equation}
This error decays \textit{exponentially} with respect to the inverse variance $1/b(t)^2$. Since the noise-agnostic estimator $\hat{\vx}(\vu)$ is a mixture of these conditional estimators, it inherits this exponential decay.
The limit becomes a competition between the polynomial divergence of the gain and the exponential convergence of the estimator:
\begin{equation}
    \lim_{t \to 0} \frac{e^{-C/b(t)^2}}{b(t)^2} = 0.
\end{equation}
Thus, $\lim_{t \to 0} \Delta \vv_{signal} = 0$, proving that signal prediction is asymptotically stable for autonomous models on discrete data manifolds.

\paragraph{Case 3: Velocity Prediction (Flow Matching)}
Target: $\vv = \dot{\vu}$ ($c=-1, d=1$).
The denominator of the unified coefficients is $ad-bc = 1$, resulting in a constant gain $\nu(t) = 1$.
The error norm is simply the posterior deviation of the velocity field:
\begin{equation}
    \Delta \vv_{FM} = \| \mathbb{E}_{\tau|\vu}[f^*_\tau(\vu)] - f^*_t(\vu) \|.
\end{equation}
Since the optimal autonomous targets $f^*$ are bounded and the gain $\nu(t)$ is unity, the error term does not contain singularities. Therefore, the Drift Perturbation Error $\Delta \vv_{FM}$ remains bounded, indicating that velocity parameterization is inherently stable for autonomous generation.

\end{document}